\documentclass[10pt,journal,compsoc]{IEEEtran}
\usepackage[table]{xcolor}

\ifCLASSOPTIONcompsoc
\usepackage[nocompress]{cite}
\else
\usepackage{cite}
\fi

\ifCLASSINFOpdf
\usepackage[pdftex]{graphicx}
\graphicspath{{./fig/}}
\usepackage{epstopdf}
\DeclareGraphicsExtensions{.pdf,.jpeg,.png,.eps,.jpg}
\else
\usepackage[dvips]{graphicx}
\graphicspath{{./fig/}}
\DeclareGraphicsExtensions{.pdf,.jpeg,.png,.eps,.jpg}
\fi

\usepackage[cmex10]{amsmath}
\usepackage{amsthm,amssymb,amsfonts,mathrsfs,bm}
\usepackage{algorithm,algorithmicx,algpseudocode}
\usepackage{ragged2e}
\usepackage{tikz}
\usepackage{multirow}
\usetikzlibrary{fit,positioning,arrows,automata}
\usepackage[inline]{enumitem}
\usetikzlibrary{shapes}
\ifCLASSOPTIONcompsoc
\usepackage[caption=false,font=footnotesize,labelfont=sf,
textfont=sf]{subfig} 
\else
\usepackage[caption=false,font=footnotesize]{subfig}
\fi

\algdef{SE}[DOWHILE]{Do}{doWhile}{\algorithmicdo}[1]{\algorithmicwhile\ #1}%

\DeclareMathOperator{\prob}{Pr}
\DeclareMathOperator{\Expect}{\mathbb{E}}

\newlength\dlf  

\theoremstyle{plain}

\newtheorem{proposition}{Proposition}

\theoremstyle{definition}
\newtheorem{remark}{Remark}

\makeatletter
\renewenvironment{proof}[1][\proofname] {\par\pushQED{\qed}\normalfont\topsep6\p@\@plus6\p@\relax\trivlist\item[\hskip\labelsep\bfseries#1\@addpunct{.}]\ignorespaces}{\popQED\endtrivlist\@endpefalse}
\makeatother
\begin{document}
	\title{{Unsupervised Ensemble Classification\\ with \textcolor{black}{Sequential and Networked} Data}}
	%
	
	\author{Panagiotis~A.~Traganitis,~\IEEEmembership{Member,~IEEE,} 
		and~Georgios~B.~Giannakis,~\IEEEmembership{Fellow,~IEEE}
		\IEEEcompsocitemizethanks{\IEEEcompsocthanksitem Panagiotis A. Traganitis and Georgios B. Giannakis are
			with the Dept.\ of Electrical and Computer Engineering and the
			Digital Technology Center, University of Minnesota,
			Minneapolis, MN 55455, USA.\protect\\   
			Work supported by NSF grants 1500713 and 1514056
			Emails: \{traga003@umn.edu, georgios@umn.edu\}
	}}
	
	%
	%
	
	\markboth{IEEE Transactions on Knowledge and Data Engineering,~Compiled \today}%
	{Traganitis \MakeLowercase{\textit{et al.}}: Unsupervised Ensemble Classification with Dependent Data}

\IEEEtitleabstractindextext{%
  \begin{abstract}\justifying 
  	Ensemble learning, the machine learning paradigm where multiple \textcolor{black}{models} are combined, has exhibited promising perfomance in a variety of tasks. The present work focuses on unsupervised ensemble classification. The term unsupervised refers to the ensemble combiner who has no knowledge of the ground-truth labels that each classifier has been trained on. While most prior works on unsupervised ensemble classification are designed for independent and identically distributed (i.i.d.) data, the present work introduces an unsupervised scheme for learning from ensembles of classifiers in the presence of data dependencies.  Two types of data dependencies are considered: sequential data and networked data whose dependencies are captured by a graph. \textcolor{black}{For both, novel moment matching and Expectation-Maximization algorithms are developed. Performance of these algorithms is evaluated on synthetic and real datasets, which indicate that knowledge of data dependencies in the meta-learner is beneficial for the unsupervised ensemble classification task.}
  \end{abstract}

  \begin{IEEEkeywords}
    Ensemble learning, unsupervised, sequential classification, crowdsourcing, dependent data.
  \end{IEEEkeywords}}

\IEEEdisplaynontitleabstractindextext

\maketitle

\IEEEpeerreviewmaketitle

\ifCLASSOPTIONcompsoc
\IEEEraisesectionheading{\section{Introduction}}
\else
\section{Introduction}
\label{sec:introduction}
\fi

\IEEEPARstart{A}{s} social networks, connected ``smart'' devices and highly accurate scientific instruments have permeated society, multiple machine learning, signal processing and data mining algorithms have been developed to process the generated data and draw inferences from them. With most of these algorithms typically designed to operate under different assumptions, combining them can be beneficial because different algorithms can complement each others strengths.

{\color{black}Ensemble learning refers to the task of combining multiple, possibly heterogeneous, machine learning models or learners~\cite{ensemblelearning,zhou_ensemble_methods}.\footnote{\textcolor{black}{The terms learner, model, and classifier will be used interchangeably throughout this manuscript.}}  In particular, ensemble classification refers to combining different classifiers~\cite{kuncheva_combining,Rokach_ensemble}.
Most popular ensemble approaches, such as boosting, bagging and stacking, have been developed for the supervised setup. Boosting methods improve an ensemble of weak learners by iteratively retraining them and properly weighting their outputs~\cite{AdaBoost}. Bagging algorithms train an ensemble of base classifiers on bootstrap samples of a dataset, and combine their results using majority voting. Stacking approaches train a \emph{meta-learner} using the outputs of base learners as features~\cite{stacking1,stacking2}.

In many cases however, labeled data may not be available, and/or one may only have access to pre-trained classifiers or human annotators, justifying the need for \emph{unsupervised} ensemble learning methods. Such a setup emerges in diverse disciplines including medicine~\cite{wright2007multidisciplinary}, biology~\cite{micsinai2012picking}, team decision making, distributed detection, and economics~\cite{forecast}, and has recently gained attention with the advent of crowdsourcing~\cite{crowdsourcing2}, as well as services such as Amazon's Mechanical Turk~\cite{MTurk}, and Figure8~\cite{fig8}, to name a few. Particularly in crowdsourcing, human annotators play the role of base classifiers of an ensemble learning model. Unsupervised ensemble classification is similar to stacking: a \emph{meta-learner} has to learn how to combine the outputs of multiple base learners, in the absence of labeled data. 
}



 Multiple algorithms attempt to address the unsupervised ensemble classification problem, and a common assumption is that the data are independent and identically distributed (i.i.d.) from an unknown distribution~\cite{dawid1979maximum,KOS,zhang2014spectral,jaffe2015estimating,traganitis2018,traganitis2017learning,traganitis_dependent}. In several cases however, additional domain knowledge may be available to the meta-learner. This domain knowledge provides information regarding the data distribution, as well as data dependencies. In this paper, two types of data dependence are considered: sequential and networked data.
Classification of sequential data arises in many natural language processing tasks such as part-of-speech tagging, named-entity recognition, and information extraction, to name a few~\cite{compseqlab}. Examples of networked data, where data correlations or dependencies are captured in a known graph, include citation, social, communication and brain networks among others. {If data do not exhibit network structure, the proposed networked data models can accommodate side information in the form of a graph to enhance the label fusion process.}

\textcolor{black}{A novel unified framework for \emph{unsupervised ensemble classification with data dependencies} is proposed. The presented methods and algorithms are built upon simple concepts from probability, as well as recent advances in tensor decompositions~\cite{sidiropoulos2017tensor} and optimization theory, that enable assessing the reliability of multiple learners, and combining their answers.} Similar to prior works for i.i.d. data, in our proposed model each \textcolor{black}{learner} has a fixed probability of deciding that a datum belongs to class $k$, given that the true class of the datum is $k'$. These probabilities parametrize the \textcolor{black}{learners}. Data dependencies are then encoded in the marginal probability mass function (pmf) of the data labels. For sequential data, the pmf of data labels is assumed to be a Markov chain, while for  networked data the pmf of data labels is assumed to be a Markov Random Field (MRF.) \textcolor{black}{ As we are operating in an unsupervised regime, it is required to assume that \textcolor{black}{learners} make decisions independent of each other. This assumption provides analytical and computational tractability. Under this assumption, the proposed methods are able to extract the probabilities that parametrize learner performance from their responses}. As an initial step \textcolor{black}{ of the proposed framework}, the moment-matching method we introduced
in~\cite{traganitis2018} for ensemble classification of i.i.d. data, is adopted to provide rough estimates for \textcolor{black}{learner} parameters. \textcolor{black}{ At the second step, our newly developed expectation maximization (EM) algorithms are employed to refine the parameters obtained from the initial moment matching step. These EM algorithms are tailored for the dependencies present in the data and produce the final label estimates.} 

The rest of the paper is organized as follows. Section~\ref{sec:prelim} states the problem, and provides preliminaries along with a brief description of the prior art in unsupervised ensemble classification for i.i.d. data. \textcolor{black}{Section~\ref{sec:recap_CI} provides an outline of moment matching and EM methods for i.i.d. data;}  Section~\ref{sec:seq} introduces the proposed approach to unsupervised ensemble classification for sequential data; while Section~\ref{sec:gen} deals with its counterpart for networked data. Section~\ref{sec:numerical_tests} presents numerical tests to evaluate our methods. Finally, concluding remarks and future research directions are given in Section~\ref{sec:conclusion}. 

\noindent\textbf{Notation:} Unless otherwise noted, lowercase bold letters, $\bm{x}$, denote
vectors, uppercase bold letters, $\mathbf{X}$, represent matrices, and
calligraphic uppercase letters, $\mathcal{X}$, stand for sets. The
$(i,j)$th entry of matrix $\mathbf{X}$ is denoted by
$[\mathbf{X}]_{ij}$; $\mathbf{X}^{\top}$ denotes the tranpose of matrix $\mathbf{X}$; $\mathbb{R}^{D}$ stands for the
$D$-dimensional real Euclidean space, $\mathbb{R}_{+}$ for the set of positive real numbers,  $\Expect[\cdot]$ for
expectation, and $\|\cdot\|$ for the $\ell_2$-norm. Underlined capital letters $\underline{X}$ denote tensors; while $[[\mathbf{A},\mathbf{B},\mathbf{C}]]_K$ is used to denote compactly a $K$ parallel factor (PARAFAC) analysis tensor~\cite{sidiropoulos2017tensor} with factor matrices $\mathbf{A} := [\bm{a}_1,\ldots,\bm{a}_K], \mathbf{B}:=[\bm{b}_1,\ldots,\bm{b}_K], \mathbf{C}:=[\bm{c}_1,\ldots,\bm{c}_K]$, that is $[[\mathbf{A},\mathbf{B},\mathbf{C}]]_K = \sum_{k=1}^{K}\bm{a}_k\circ\bm{b}_k\circ\bm{c}_k$. Finally, $\mathcal{I}(A)$ denotes the indicator function of event $A$, i.e. $\mathcal{I}(A) = 1$ if $A$ occurs and is $0$ otherwise. 
\section{Problem Statement and Preliminaries}\label{sec:prelim}

Consider a dataset consisting of $N$ data (possibly vectors) $\{x_n\}_{n=1}^{N}$ each belonging to one of $K$ possible classes with corresponding labels $\{y_n\}_{n=1}^{N}$, e.g. $y_n = k$ if $x_n$ belongs to class $k$. The pairs $\left\{(x_n,y_n)\right\}_{n=1}^{N}$ are drawn from an unknown joint distribution $\rm{D}$, and $X$ and $Y$ denote random variables such that $(X,Y)\sim\rm{D}$. Consider now $M$ \textcolor{black}{pre-trained learners} that observe $\{x_n\}_{n=1}^{N}$, and provide estimates of labels. Let $f_m(x_n)\in\{1,\ldots,K\}$ denote the label assigned to datum $x_n$ by the $m$-th \textcolor{black}{learner}. All \textcolor{black}{learner} responses are then collected at a  meta-learner. Collect all \textcolor{black}{learner} responses in the $M\times N$ matrix $\mathbf{F}$, that has entries $[\mathbf{F}]_{mn} = f_m(x_n)$, and all ground-truth labels in the $N\times 1$ vector $\bm{y} = [y_1,\ldots,y_N]^{\top}$ . 
The task of \emph{unsupervised ensemble classification} is: Given only the \textcolor{black}{learner} responses $\{f_m(x_n), m=1,\ldots,M\}_{n=1}^{N}$, we wish to estimate the ground-truth labels of the data $\{y_n\}$; see Fig.~\ref{fig:ensemble_class}. \textcolor{black}{In contrast to supervised ensemble classification, here the combined models (learners) have been trained beforehand. In addition, the meta-learner does not have access to any ground-truth data to learn a combining function. Thus, the focus of this work is on judiciously combining learner responses. } \textcolor{black}{This setup is similar to crowdsourced  classification, which also has the additional challenge that human annotators (corresponding to the learners in the unsupervised ensemble model) may not provide responses for all $N$ data.}

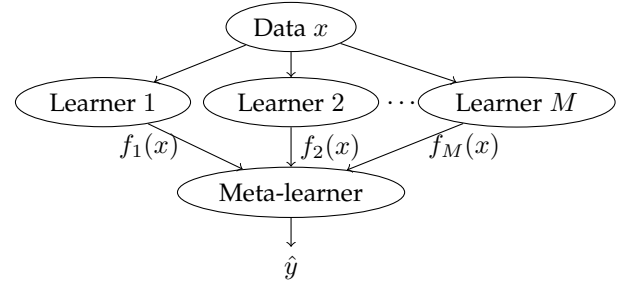
\begin{figure}
	\centering
	\begin{tikzpicture}
	\node[shape=ellipse,draw=black] (data) at (2.5,2) {Data ${x}$};
	\node[shape=ellipse,draw=black] (learner1) at (0,1) {Learner $1$};
	\node[shape=ellipse,draw=black] (learner2) at (2.5,1) {Learner $2$};
	\node[draw=none,fill=none] (dots) at (4,1) {\ldots};
	\node[shape=ellipse,draw=black] (learnerM) at (5.5,1) {Learner $M$};
	\node[shape=ellipse,draw=black] (Metalearner) at (2.5,-0.2) {Meta-learner};
	\node[draw=none,fill=none] (output) at (2.5,-1.2) {$\hat{y}$};

	\path [->] (data) edge node[left] {} (learner1);
	\path [->] (data) edge node[left] {} (learner2);
	\path [->] (data) edge node[right] {} (learnerM);
	\path [->] (learner1) edge node[left] {$f_1({x})$~~} (Metalearner);
	\path [->] (learner2) edge node[right] {$f_2({x})$} (Metalearner);
	\path [->] (learnerM) edge node[right] {~~$f_M({x})$} (Metalearner);
	\path [->] (Metalearner) edge node[right] {} (output);
	\end{tikzpicture}
	\caption{Unsupervised ensemble classification setup, where the outputs of learners are combined in parallel.} \label{fig:ensemble_class}
\end{figure}

{\color{black} Consider that each learner has a fixed probability of deciding that a datum belongs to class $k'$, when presented with a datum of class $k$; and all classifiers behavior is presumed invariant throughout the dataset. Clearly, the performance of each learner $m$ is then characterized by the so-called \emph{confusion} matrix $\mathbf{\Gamma}_m$, whose $(k',k)$-th entry is
	\begin{equation}
	\label{eq:Psi_element}
	[\mathbf{\Gamma}_m]_{k'k} := \Gamma_m(k',k) = \prob\left(f_m(X) = k' | Y = k\right).
	\end{equation}
	The $K\times K$ matrix $\mathbf{\Gamma}_m$ has non-negative entries that obey the simplex constraint, since $\sum_{k'=1}^{K}\prob\left(f_m(X) = k' | Y = k\right) = 1$, for $k=1,\ldots,K$; hence, entries of each $\mathbf{\Gamma}_m$ column sum up to $1$, that is, $\mathbf{\Gamma}_m^{\top}\bm{1} = \bm{1}$ and $\mathbf{\Gamma}_m \geq \bm{0}$. The confusion matrix showcases the statistical behavior of a learner, as each column provides the learners's probability of deciding the correct class, when presented with a datum from each class.

	Before proceeding, we adopt the following assumptions.
	\renewcommand{\labelenumi}{\textbf{As\arabic{enumi}.}}
	\begin{enumerate}
		\item \label{ass:annot_indep}
		Responses of different learners per datum are conditionally independent, given the ground-truth label $Y$ of the same datum $X$; that is, 
		\begin{alignat*}{1}
		& \prob\left(f_1(X)=k_1,\ldots,f_M(X)=k_M | Y=k\right)  \\ &= \prod_{m=1}^{M}\prob\left(f_m(X)=k_m | Y=k\right) = \prod_{m=1}^{M}\Gamma_m(k_m,k)
		\end{alignat*}
		\item \label{ass:better_than_random}
		Most learners are better than random.
	\end{enumerate}
	
	As\ref{ass:annot_indep} is satisfied by learners making decisions independently, which is a rather standard assumption~\cite{dawid1979maximum,jaffe2015estimating,zhang2014spectral}. As1 can be thought of as the manifestation of diversity, which is sought after in ensemble learning systems~\cite{zhou_ensemble_methods}. Graphically, this model is depicted in Fig.~\ref{fig:ds_model}. 
	Under As2, the largest entry per $\mathbf{\Gamma}_m$ column is the one on the diagonal
	\begin{equation*}
	[\mathbf{\Gamma}_m]_{kk}\geq [\mathbf{\Gamma}_m]_{k'k}, \text{ for } k',k=1,\ldots,K.
	\end{equation*}
	As\ref{ass:better_than_random} is another standard assumption, used to alleviate the inherent permutation ambiguity in estimating $\mathbf{\Gamma}_m$.
	
	Based on the aforementioned model, knowledge of the structure of the joint pdf of learner responses and labels is critical.
	When data $\left\{(x_n,y_n)\right\}_{n=1}^{N}$ are i.i.d. the joint pdf can be expressed as a product, that is
	\begin{equation}
	\prob(\mathbf{F},\bm{y}) = \prod_{n=1}^{N}\prob(\mathbf{f}_n,y_n)
	\end{equation}
    where $\mathbf{f}_n := [f_1(x_n),\ldots,f_M(x_n)]^{\top}.$
}
	\begin{figure}[tb]
		\centering
		\begin{tikzpicture}
		\tikzstyle{main}=[ellipse, minimum size = 10mm, thick, draw =black!80, node distance = 8mm]
		\tikzstyle{connect}=[-latex, thick]
		\tikzstyle{box}=[rectangle, draw=black!100]
		
		\node[main] (Y) [] {$Y$};
		
		\node[main,fill=black!10] (O2) [below=of Y] {$f_2(X)$};
		\node[main,fill=black!10] (O1) [left=of O2] {$f_1(X)$};
		\node[main,fill=black!10] (OM) [right=of O2] {$f_M(X)$};
		\path (O2) -- node[auto=false]{\ldots} (OM);
		\path (Y) edge [connect] (O1);
		\path (Y) edge [connect] (O2);
		\path (Y) edge [connect] (OM);
		
		\end{tikzpicture}
		\caption{Graphical depiction of the Dawid-Skene model for i.i.d. data. Shaded ellipses are observed variables (here \textcolor{black}{classifier} responses).}
		\label{fig:ds_model}
	\end{figure}
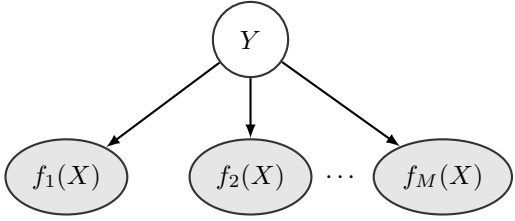

{\color{black} Consider now that data $\left\{(x_n,y_n)\right\}_{n=1}^{N}$  belong to a sequence, i.e. the $n$th datum depends on the $(n-1)$st one. Pertinent settings emerge with speech and natural language processing tasks such as word identification, part-of-speech tagging, named-entity recognition, and information extraction~\cite{compseqlab}.
	
	In order to take advantage of this dependence, we will encode it in the marginal pmf $\prob(\bm{y})$ of the labels. Specifically, we postulate that the sequence of labels $\{y_n\}_{n=1}^{N}$ forms a one-step time-homogeneous Markov chain; that is, variable $y_{n}$ depends only on its immediate predecessor $y_{n-1}$. This Markov chain is characterized by a $K\times K$ transition matrix $\mathbf{T}$, whose $(k,k')$-th entry is given by
	\begin{equation*}
	[\mathbf{T}]_{kk'} = T(k,k')= \prob(y_n = k | y_{n-1} = k').
	\end{equation*}
	Matrix $\mathbf{T}$ has non-negative entries that satisfy the simplex constraint. The marginal probability of  $\{y_n\}_{n=1}^{N}$ can be expressed using successive conditioning, as
	\begin{align}
	\label{eq:markovchain}
	& \prob\left(\bm{y} = \bm{k}\right)=   \prob(y_1 = k_1)\prod_{n=2}^{N}T(k_n,k_{n-1})
	\end{align}
	where $\bm{k} := [k_1,\ldots,k_N]^{\top}$. Accordingly, the data $\{x_n \}_{n=1}^{N}$ depend only on their corresponding $y_n$, and can be drawn from an unknown conditional pdf as $x_n\sim \prob(x_n|y_n = k_n)$. The data pairs $\{({x}_n,y_n)\}_{n=1}^{N}$ form a hidden Markov model (HMM), where the labels $\{ y_n\}_{n=1}^{N}$ correspond to the hidden variables of the HMM, while $\{{x}_n\}_{n=1}^{N}$ correspond to the observed variables of the HMM.
	As with the i.i.d. case, $M$ learners observe $\{x_n\}_{n=1}^{N}$, and provide estimates of their labels $f_m(x_n)$. Under As\ref{ass:annot_indep}, the responses of different learners per datum are conditionally independent, given the ground-truth label $y_n$ of the same datum $x_n$.
    A graphical representation of this HMM is provided in Fig.~\ref{fig:hmm_model}.

Another type of data dependency considered in this work, is the dependency between networked data. 
In many cases, additional information pertaining to the data is available in the form of an undirected graph $\mathcal{G}(\mathcal{V},\mathcal{E})$, where $\mathcal{V}$ and $\mathcal{E}$ denote the vertex (or node) and edge sets of $\mathcal{G}$, respectively. Examples of such networked data include social and citation networks~\cite{cora_citeseer,pubmed}. Each node of this graph corresponds to a data point, thus $|\mathcal{V}| = N$, while the edges capture pairwise relationships between the data.

{For networked data, we will take a similar approach to sequential data and encode data dependence, meaning the pairwise relationships provided by the graph $\mathcal{G}$, in the marginal pmf $\prob(\bm{y})$.}
Specifically, we model the labels $\{y_n\}_{n=1}^{N}$ as being drawn from an MRF.
Due to the MRF structure of $\prob(\bm{y})$, the conditional pmf of $y_n$, for all $n=1,\ldots,N$, satisfies the local Markov property
\begin{equation}
\prob(y_n | \bm{y}_{-n}) = \prob(y_n | \bm{y}_{\mathcal{N}_n}) 
\end{equation}
where $\bm{y}_{-n}$ is a vector containing all labels except $y_n$ and $\bm{y}_{\mathcal{N}_n}$ is a vector containing the labels of node $n$ neighbors. Then, the joint pmf of all labels is
\begin{equation}
\label{eq:MRF_prob}
\prob(\bm{y}) = \frac{1}{Z}\text{exp}(-U(\bm{y}))
\end{equation}
where $Z := \sum_{\bm{y}}\text{exp}(-U(\bm{y}))$ is the normalization constant, and $U(\bm{y})$ is the so-called energy function. Computing the normalization constant $Z$ involves all possible configurations of $\bm{y}$; hence, it is intractable for datasets with moderate to large size $N$.  By the Hammersley-Clifford theorem the energy function can be written as~\cite{HammersleyClifford}
\begin{equation}
U(\bm{y}) = \frac{1}{2}\sum_{(n,n')\in\mathcal{E}}V(y_n,y_{n'})
\end{equation}
where $V(y_n,y_{n'})$ denotes the clique potential of the $(n,n')$-th edge, which will be defined in the following sections.

Similar to the i.i.d. and sequential cases, data vectors $\{x_n\}_{n=1}^{N}$ are generated from an unknown conditional pdf $x_n\sim\prob(x_n|y_n = k)$ that depends only on their corresponding label $y_n$. $M$ learners observe $\{x_n\}_{n=1}^{N}$ and provide estimates of their labels $f_m(x_n)$.
With As\ref{ass:annot_indep} still in effect, learner responses are conditionally independent given $Y$.

}

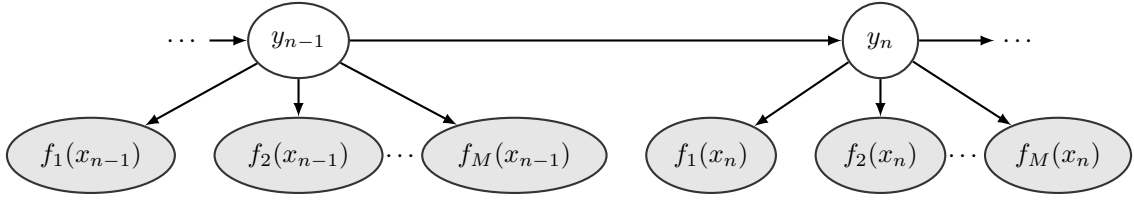
\begin{figure*}[tb]
	\centering
	\begin{tikzpicture}
	\tikzstyle{main}=[ellipse, minimum size = 10mm, thick, draw =black!80, node distance = 5mm]
	\tikzstyle{connect}=[-latex, thick]
	\tikzstyle{box}=[rectangle, draw=black!100]
	\node[box,draw=white!100] (beginpart) {\ldots};
	\node[main] (L1) [right=of beginpart] {$y_{n-1}$};

	
	\node[main,fill=black!10] (O12) [below=of L1] {$f_2(x_{n-1})$};
	\node[main,fill=black!10] (O11) [left=of O12] {$f_1(x_{n-1})$};
	\node[main,fill=black!10] (O1M) [right=of O12] {$f_M(x_{n-1})$};
	\path (O12) -- node[auto=false]{\ldots} (O1M);
	\node[main,fill=black!10] (O21) [right=of O1M] {$f_1(x_n)$};
	\node[main,fill=black!10] (O22) [right=of O21] {$f_2(x_n)$};
	\node[main,fill=black!10] (O2M) [right=of O22] {$f_M(x_n)$};
	\path (O22) -- node[auto=false]{\ldots} (O2M);
	\node[main] (L2) [above=of O22] {$y_n$};
	\node[box,draw=white!100] (endpart) [right=of L2] {\ldots};
	\path (beginpart) edge [connect] (L1);
	\path (L2) edge [connect] (endpart);
	\path (L1) edge [connect] (L2);
	\path (L1) edge [connect] (O11);
	\path (L1) edge [connect] (O12);
	\path (L1) edge [connect] (O1M);
	\path (L2) edge [connect] (O21);
	\path (L2) edge [connect] (O22);
	\path (L2) edge [connect] (O2M);
	\end{tikzpicture}
	\caption{Graphical representation of the proposed model for sequential data. Shaded ellipses indicate observed variables (here \textcolor{black}{learner} responses).}
	\label{fig:hmm_model}
\end{figure*}

\subsection{Prior work}\label{ssec:prior}
Most prior works on unsupervised ensemble classification focus on the i.i.d. data case.
Possibly the simplest scheme is majority voting, where the estimated label of a datum is the one that most \textcolor{black}{learner}s agree upon. Majority voting presumes that all \textcolor{black}{learner}s are equally ``reliable,'' which is rather unrealistic, both in crowdsourcing as well as in ensemble learning setups. Other unsupervised ensemble methods aim to estimate the parameters that characterize the \textcolor{black}{learner}s' performance, namely the \textcolor{black}{learner} confusion matrices. A popular approach is joint maximum likelihood (ML) estimation of the unknown labels and the aforementioned \textcolor{black}{confusion matrices} using the expectation-maximization (EM) algorithm~\cite{dawid1979maximum}. As EM iterations do not guarantee convergence to the ML solution, recent works pursue alternative estimation methods. For binary classification, \cite{EigenRatio,KOS}  \textcolor{black}{describe each learner with one parameter, its probability of providing a correct response, and attempts to learn this parameter.} Recently, \textcolor{black}{for binary classification,} \cite{jaffe2015estimating} \textcolor{black}{used the eigendecomposition of the learner crosscovariance matrix, to estimate the entries of the $2\times2$ confusion matrices of learners, while~\cite{minimaxopt} introduced a minimax optimal algorithm that can infer learner reliabilities.}. In the multiclass setting, spectral methods such as~\cite{zhang2014spectral,traganitis2018} utilize third-order moments and tensor decompositions to estimate the unknown reliability parameters. These estimates can then initialize the EM algorithm of~\cite{dawid1979maximum}. 

Recent works advocate unsupervised ensemble approaches for sequential data. A method to aggregate \textcolor{black}{learner} labels for sequential data relies on conditional random fields (CRFs) \cite{Rodrigues2014}. However, this method operates under strong and possibly less realistic assumptions requiring e.g., that only one \textcolor{black}{learner} provides the correct label for each datum. To relax the assumptions of~\cite{Rodrigues2014}, extensions of the standard Hidden Markov Model (HMM) to incorporate \textcolor{black}{learner} responses have been reported along with a variational EM~\cite{variational} algorithm to aggregate them~\cite{pmid_hmm_em}. As both aforementioned methods require tuning of hyperparameters, a  training step is necessary, which can be unrealistic in unsupervised settings. 
 
When features $\{x_n\}_{n=1}^{N}$ are available at the meta-learner, approaches based on Gaussian Processes can be used to classify the data based on \textcolor{black}{learner} responses~\cite{Yan2014,vgp,musicgenre_senpoldata}. These approaches can take advantage of some data dependencies, \textcolor{black}{as they rely on linear or nonlinear similarities between data};  however, in addition to requiring the data features at the meta-learner, these methods have only been developed for binary classification.

The present work puts forth a novel scheme for \emph{unsupervised ensemble classification in the presence of data dependencies}. Our approach builds upon our previous work on unsupervised ensemble classification of i.i.d. data~\cite{traganitis2018,traganitis2017learning}, and markedly extends its scope to handle sequential as well as networked data, without requiring training or access to data features. For sequentially dependent data case we present a moment matching algorithm that is able to estimate \textcolor{black}{learner} confusion matrices as well as the parameters characterizing the Markov chain of the labels. These confusion matrices and parameters are then refined using  EM iterations tailored for the sequential data classification task. For the network dependencies, the moment matching method for i.i.d. data is used to initialize an EM algorithm designed for networked data. To our knowledge, no existing work tackles the ensemble classification task under the networked data regime. Compared to our conference precursor in~\cite{traganitis_ensemble_dep_dsw}, here we have included extensive numerical tests as well as a new EM algorithm for sequential data along with algorithms that tackle networked data.

{\color{black}
\section{Unsupervised Ensemble Classification of i.i.d. data}
\label{sec:recap_CI}
Before introducing algorithms for sequential and networked data, we first outline moment matching and EM approaches for i.i.d. data. } If all \textcolor{black}{learner} confusion matrices were ideally known, the label of $x_n$ could be estimated using a maximum a posteriori (MAP) classifier. The latter finds the label $k$ that maximizes the joint probability of $y_n$ and observed \textcolor{black}{learner} responses $\{f_m(x_n)=k_m \}_{m=1}^{M}$
\begin{align}
	\label{eq:max_label_bayes_log}
	\hat{y}_n & = \underset{{k\in\{1,\ldots,K\}}}{\arg\max} {{L}}(x_n|k)\prob(y_n = k) \\
	& = \underset{{k\in\{1,\ldots,K\}}}{\arg\max} \log\pi_k + \sum_{m=1}^{M}\log(\Gamma_m(k_m,k)) \notag 
\end{align}
where $\pi_k := \prob(y_n = k) = \prob(Y = k)$ and ${L}(x_n|k):=\prob\left(f_1(x_n)=k_1,\ldots,f_M(x_n)=k_M|Y = k\right)$, and the second equality of \eqref{eq:max_label_bayes_log} follows from As\ref{ass:annot_indep} and properties of the logarithm. In addition, if all classes are assumed equiprobable, \eqref{eq:max_label_bayes_log} reduces to the maximum likelihood (ML) classifier. 

But even for non-equiprobable classes, unsupervised ensemble classification requires estimates of the class priors $\bm{\pi}:=[\pi_1,\ldots,\pi_K]^{\top}$ as well as all the \textcolor{black}{learner} performance parameters $\{ \mathbf{\Gamma}_m\}_{m=1}^{M}$. 

\subsection{EM algorithm for i.i.d. data}
\label{sssec:EM_iid}
Here we outline how the EM algorithm can be employed to estimate the wanted \textcolor{black}{learner} performance parameters by iteratively maximizing the log-likelihood of the observed \textcolor{black}{learner} responses; that is, $\log\prob(\mathbf{F} ; \bm{\theta})$,
where $\bm{\theta}$ collects all the \textcolor{black}{learner} confusion matrices (prior probabilities are assumed equal and are dropped for simplicity). Each EM iteration includes the expectation (E-)step and the maximization (M-)step. 

At the E-step of the $(i+1)$st iteration, the estimate $\bm{\theta}^{(i)}$ is given, and the so-termed Q-function is obtained as  
\begin{align}
& Q(\bm{\theta} ; \bm{\theta}^{(i)}) = \Expect_{\bm{y}|\mathbf{F};\bm{\theta}^{(i)}}[\log\prob(\bm{y},\mathbf{F};\bm{\theta})] \\ & = \Expect_{\bm{y}|\mathbf{F};\bm{\theta}^{(i)}}[\log\prob(\mathbf{F}|\bm{y};\bm{\theta})] + \Expect_{\bm{y}|\mathbf{F};\bm{\theta}^{(i)}}[\log\prob(\bm{y};\bm{\theta})] \notag.
\end{align}
Since the data are i.i.d., we have under As\ref{ass:annot_indep} that
\begin{align*}
 \Expect_{\bm{y}|\mathbf{F};\bm{\theta}^{(i)}}[\log\prob(\mathbf{F}|\bm{y};\bm{\theta})] =  \sum_{n=1}^{N}\sum_{k=1}^{K}\sum_{m=1}^{M}\log\Gamma_m(f_m(x_n),k)q_{nk}
\end{align*}
where $q_{nk} := 
\prob(y_n = k | \{ f_m(x_n) \}_{m=1}^{M}; \bm{\theta}^{(i)})$ is the posterior of label $y_n$ given the observed data and current parameters, and 
\begin{equation*}
\Expect_{\bm{y}|\mathbf{F};\bm{\theta}^{(i)}}[\log\prob(\bm{y};\bm{\theta})] = \sum_{n=1}^{N}\sum_{k=1}^{K}\log\prob(y_n = k;\bm{\theta})q_{nk} .
\end{equation*}
Under this model, it can be shown that~\cite{dawid1979maximum,zhang2014spectral}
\begin{align}
q_{nk}^{(i+1)} =  \frac{1}{Z}\exp\left(\sum_{m=1}^{M}\sum_{k'=1}^{K}\mathcal{I}(f_m(x_n) =  k')\log(\Gamma_m^{(i)}(k',k))\right) 
\end{align}
where $Z$ is the normalization constant.

At the M-step, \textcolor{black}{learner} confusion matrices are updated by maximizing the Q-function to obtain \begin{equation}
\label{eq:thetaupd_iid}
\bm{\theta}^{(i+1)} = \arg\max_{\bm{\theta}} Q(\bm{\theta} ; \bm{\theta}^{(i)}).
\end{equation}
It can be shown per \textcolor{black}{learner} $m$ that \eqref{eq:thetaupd_iid} boils down to
\begin{equation}
\label{eq:EM_gammaupd}
[\mathbf{\Gamma}_m^{(i+1)}]_{k'k} = \frac{\sum_{n=1}^{N}q_{nk}^{(i+1)}\mathcal{I}(f_m(x_n) = k')}{\sum_{k^{''}=1}^{K}\sum_{n=1}^{N}q_{nk}^{(i+1)}\mathcal{I}(f_m(x_n) = k^{''})}.
\end{equation} 
The E- and M-steps are repeated until convergence, and ML estimates of the labels 
are subsequently found as 
\begin{align*}
\hat{y}(x_n) & = \underset{k\in\{1,\ldots,K\}}{\arg\max} \prob(\{f_m(x_n)\}_{m=1}^{M}, y_n = k) \\
& = \underset{k\in\{1,\ldots,K\}}{\arg\max} q_{nk}.
\end{align*}
\subsection{Moment-matching for i.i.d. data}
\label{sssec:MM_iid}
As an alternative to EM, the annonator performance parameters can be estimated using the moment-matching method we introduced for i.i.d. data in~\cite{traganitis2018}, which we outline here before extending it to dependent data in the ensuing sections.

Consider representing label $k$ using the canonical $K\times 1$ vector $\bm{e}_k$, namely the $k$-th column of the $K\times K$ identity matrix $\mathbf{I}$. Let $\mathbf{f}_m(X)$ denote the $m$-th \textcolor{black}{learner}'s response in vector format. Since $\mathbf{f}_m(X)$ is just a vector representation of $f_m(X)$, it holds that $\prob\left(f_m(X) = k' | Y = k\right)\equiv\prob\left(\mathbf{f}_m(X) = \bm{e}_{k'} | Y = k\right)$. With $\bm{\gamma}_{m,k}$ denoting the $k$-th column of $\mathbf{\Gamma}_m$, it thus holds that
\begin{equation}
\label{eq:cond_exp_f_i}
\bm{\gamma}_{m,k} := \Expect[\mathbf{f}_m(X)|Y=k] =\sum_{k'=1}^{K} \bm{e}_{k'}\prob\left(f_m(X)=k'|Y=k\right)
\end{equation}
where we used the definition of conditional expectation.
Using \eqref{eq:cond_exp_f_i} and the law of total probability, the mean vector of responses from \textcolor{black}{learner} $m$, is hence 
\begin{equation}
\label{eq:mean_annotator_response}
\hspace{-0.8pt}\Expect[\mathbf{f}_m(X)]  =  \sum_{k=1}^{K} \Expect[\mathbf{f}_m(X)|Y=k]\prob\left(Y=k\right)  = \mathbf{\Gamma}_m\bm{\pi}.
\end{equation}
Upon defining the diagonal matrix $\mathbf{\Pi} := \text{diag}(\bm{\pi})$, the $K\times K$ cross-correlation matrix between the responses of \textcolor{black}{learner}s $m$ and $m'\neq m$, can be expressed as
\begin{alignat}{2}
& \mathbf{R}_{mm'} &&:= \Expect[\mathbf{f}_m(X)\mathbf{f}_{m'}^{\top}(X)] \notag\\& &&=
\sum_{k=1}^{K} \Expect[\mathbf{f}_m(X)|Y=k]\Expect[\mathbf{f}_{m'}^{\top}(X)|Y=k]\prob\left(Y=k\right) \notag
\\ & && =\mathbf{\Gamma}_m\text{diag}(\bm{\pi})\mathbf{\Gamma}_{m'}^{\top} = \mathbf{\Gamma}_m\mathbf{\Pi}\mathbf{\Gamma}_{m'}^{\top}
\label{eq:crossx_annotator_response}
\end{alignat}
where we successively relied on the law of total probability, As\ref{ass:annot_indep}, and \eqref{eq:cond_exp_f_i}. Consider now the $K\times K\times K$ cross-correlation tensor between the responses of \textcolor{black}{learner}s $m$, $m'\neq m$ and $m''\neq m', m$, namely
\begin{alignat}{2}
\label{eq:crossx_tensor}
\underline{\Psi}_{mm'm''} = \Expect[{\bf{f}}_m(X)\circ{\bf{f}}_{m'}(X)\circ{\bf{f}}_{m''}(X)].
\end{alignat}
It can be shown that $\underline{\Psi}_{mm'm''}$ obeys a PARAFAC model with factor matrices $\mathbf{\Gamma}_m,\mathbf{\Gamma}_{m'}$ and $\mathbf{\Gamma}_{m''}$~\cite{traganitis2018};
that is,
\begin{alignat}{2}
\label{eq:crossx_tensor2}
\underline{\Psi}_{mm'm''} &= \sum_{k=1}^{K}\pi_k\bm{\gamma}_{m,k}\circ\bm{\gamma}_{m',k}\circ\bm{\gamma}_{m'',k}  \\ &= [[\mathbf{\Gamma}_m\mathbf{\Pi},\mathbf{\Gamma}_{m'},\mathbf{\Gamma}_{m''}]]_K.\notag
\end{alignat}
Note here that the diagonal matrix $\mathbf{\Pi}$ can multiply any of the factor matrices $\mathbf{\Gamma}_m,\mathbf{\Gamma}_{m'}$, or, $\mathbf{\Gamma}_{m''}$.

Having available the sample average counterparts of \eqref{eq:mean_annotator_response}, \eqref{eq:crossx_annotator_response} and~\eqref{eq:crossx_tensor}, correspondingly $\{\bm{\mu}_m\}_{m=1}^{M}$, $\{\mathbf{S}_{mm'}\}_{m,m'=1}^{M}$, and $\{\underline{T}_{mm'm^{''}}\}_{m,m',m^{''}=1}^{M}$, estimates of $\{\mathbf{\Gamma}_m \}_{m=1}^{M}$ and $\bm{\pi}$ can be readily obtained. This approach is an instantiation of the method of moments estimation method, see e.g.~\cite{kay1993fundamentals_est}, and can be cast as the following constrained optimization task
\begin{alignat}{2}
\label{eq:ConfMat_Opt_crosscov}
& \underset{{\bm{\pi}\in\mathcal{C}_\pi,}{\{\mathbf{\Gamma}_m\in\mathcal{C} \}_{m=1}^{M}}}{\min} && g(\{\mathbf{\Gamma}_m \}_{m=1}^{M_{}},\bm{\pi}) 
\end{alignat}
where
\vspace{-8pt}
{\small \begin{alignat*}{2}
	& g(\{\mathbf{\Gamma}_m\},\bm{\pi}) :=  \sum_{m=1}^{M}\|\bm{\mu}_m - \mathbf{\Gamma}_m\bm{\pi}\|_2^2 \\ &+ \sum_{\underset{m'>m}{m=1}}^{M}\|\mathbf{S}_{mm'} - \mathbf{\Gamma}_m\mathbf{\Pi}{\mathbf{\Gamma}_{m'}^{\top}}\|_F^2 \\ 
	& + \sum_{\underset{{m'>m,m''>m'}}{m=1}}^{M}\|\underline{T}_{mm'm''} - [[\mathbf{\Gamma}_{m}\mathbf{\Pi},\mathbf{\Gamma}_{m'},\mathbf{\Gamma}_{m''}]]_K\|_F^2,
	\end{alignat*}}
\textcolor{black}{\noindent
	 $\mathcal{C} := \{\mathbf{\Gamma}\in\mathbb{R}^{K\times K}: \mathbf{\Gamma}\geq \bm{0}, \mathbf{\Gamma}^{\top}\bm{1} = \bm{1}\}$, is the convex set of matrices whose columns lie on a probability simplex, and $\mathcal{C}_{p} := \{\bm{u}\in\mathbb{R}^{K}: \bm{u}\geq \bm{0},\bm{u}^{\top}\bm{1} = 1\}$ denotes the simplex constraint for a $K\times 1$ vector.}
The non-convex optimization in \eqref{eq:ConfMat_Opt_crosscov} can be solved using the alternating optimization method described in~\cite{traganitis2018}, which is guaranteed to converge to a stationary point of $g$~\cite{aoadmm}. As\ref{ass:better_than_random} is used here to address the permutation ambiguity that is induced by the tensor decomposition of \eqref{eq:ConfMat_Opt_crosscov}. Interested readers are referred to \cite{traganitis2018} for implementation details, and theoretical guarantees.

Upon obtaining $\{\hat{\mathbf{\Gamma}}_m \}_{m=1}^{M}$ and  $\hat{\bm{\pi}}$, a MAP classifier can be subsequently employed to estimate the label for each datum; that is, for $n=1,\ldots,N$, we obtain
\begin{equation}
\label{eq:label_est}
\hat{y}_{}(x_n) = \underset{{k\in\{1,\ldots,K\}}}{\arg\max} \log\hat{\pi}_k + \sum_{m=1}^{M}\log\hat{{\Gamma}}_m({f_m(x_n),k})
\end{equation}
where $\hat{{\Gamma}}_m({k',k}) = [\hat{\mathbf{\Gamma}}_m]_{k'k}$, and $\hat{\pi}_k = [\hat{\bm{\pi}}]_k$.
The estimates $\{\hat{\mathbf{\Gamma}}_m \}_{m=1}^{M}, \hat{\bm{\pi}}, \{\hat{y}_n\},$ can be improved using the EM iterations in Sec.~\ref{sssec:EM_iid}. Such a refinement is desirable when $N$ is relatively small, and thus moment estimates are not as reliable.

Next, we will introduce our novel approaches for ensemble classification with sequential and networked data.

\section{Sequentially Dependent Data}
\label{sec:seq}
{\color{black} Having recapped moment matching and EM approaches for i.i.d. data, we now turn our attention to the sequential data case. Recall from Sec.~\ref{sec:prelim} that we postulate labels $\bm{y}$ forming a one-step time-homogeneous Markov chain, characterized by the transition matrix $\mathbf{T}\in\mathcal{C}$, and that learner responses obey As\ref{ass:annot_indep}. The labels $\bm{y}$ along with learner responses $\mathbf{F}$, form an HMM. As with the i.i.d. case, here we develop both moment-matching and EM algorithms tailored for sequential data. }

\subsection{Label estimation for sequential data}
\label{ssec:Viterbi}
Given only \textcolor{black}{learner} responses for all data in a sequence, an approach to estimating the labels of each datum, meaning the hidden variables of the HMM, is to find the sequence $\bm{k}$ that maximizes the joint probability of the labels  $\bm{y}$ and the \textcolor{black}{learner} responses $\mathbf{F}$, namely
\begin{align}
	& \prob\left(\bm{y} = \bm{k},\mathbf{F}\right) \notag \\
	& = \prob(y_1 = k_1)\prod_{n=2}^{N}T(k_n,k_{n-1})\prod_{m=1}^{M}\Gamma_m(f_m(x_n),k_n)\label{eq:label_est_seq}
\end{align}
where the equality is due to \eqref{eq:markovchain} and As\ref{ass:annot_indep}.
This can be done efficiently using the Viterbi algorithm~\cite{hmmtutor,viterbi}.
In order to obtain estimates of the labels, $\{\mathbf{\Gamma}_m\}_{m=1}^{M}$ and $\mathbf{T}$ must be available. The next subsections will show that $\{\mathbf{\Gamma}_m\}_{m=1}^{M}$ and $\mathbf{T}$ can be recovered by the \textcolor{black}{learner responses, using moment matching and/or EM approaches.}

\subsection{\textcolor{black}{Moment matching for sequential data}}
\label{ssec:Tmatest}
{\color{black} 
Under the sequential data model outlined earlier, we require an additional assumption before presenting our moment matching algorithm. 
\renewcommand{\labelenumi}{\textbf{As\arabic{enumi}.}}
\begin{enumerate}[wide, labelwidth=!, labelindent=0pt]
	\setcounter{enumi}{2}
	\item \label{ass:mixing}
	The Markov chain formed by the labels $\{y_n\}$ has a unique stationary distribution $\bm{\pi} := [\pi_1,\ldots,\pi_K]^{\top} = [\prob(Y=1),\ldots,\prob(Y=K)]^{\top}$, and is also irreducible.
\end{enumerate}
Similar to~\cite{kontorovich13}, this assumption enables decoupling the problem of learning the parameters of interest in two steps. First, estimates of the confusion matrices $\{\hat{\mathbf{\Gamma}}_m \}_{m=1}^{M}$ and stationary priors $\hat{\bm{\pi}}$ are obtained; and subsequently, the transition matrix is estimated as $\hat{\mathbf{T}}$ before obtaining an estimate of the labels $\{\hat{y}_n\}_{n=1}^{N}$.

}
Under As\ref{ass:mixing}, the HMM is mixing and assuming that $y_0$ is drawn from the stationary distribution $\bm{\pi}$,  the responses of a \textcolor{black}{learner} $m$ can be considered to be generated from a mixture model, see e.g.~\cite{kontorovich13}
\begin{equation}
f_m(X)\sim  \sum_{k=1}^{K}\pi_k\prob(f_m(X)|Y=k)\;.
\end{equation}
Based on the latter, the remainder of this subsection will treat labels $\{y_n\}_{n=1}^{N}$, as if they had been drawn i.i.d. from the stationary distribution $\bm{\pi}$, that is $y_n\sim\bm{\pi}$ for $n=1,\ldots,N$. Then, the procedure outlined in Sec.~\ref{sssec:MM_iid} can be readily adopted to obtain estimates of the stationary distribution $\hat{\bm{\pi}}$ and the confusion matrices $\{\hat{\mathbf{\Gamma}}_m \}_{m=1}^{M}$.

With estimates of \textcolor{black}{learner} confusion matrices $\{\hat{\mathbf{\Gamma}}_m\}$ and stationary probabilities $\hat{\bm{\pi}}$ at hand, we turn our attention to the estimation of the transition matrix $\mathbf{T}$. To this end, consider the cross-correlation matrix of consecutive vectorized observations between \textcolor{black}{learner}s $m$ and $m'$, namely $\tilde{\mathbf{R}}_{mm'} = \Expect[\mathbf{f}_m(x_n)\mathbf{f}_{m'}^{\top}(x_{n-1})]$. Under the HMM of Sec.~\ref{sec:seq}, the latter is given by
\begin{equation}
	\label{eq:crossx_T}
	\tilde{\mathbf{R}}_{mm'} = \mathbf{\Gamma}_m\mathbf{T}\text{diag}(\bm{\pi})\mathbf{\Gamma}_{m'}^{\top} = \mathbf{\Gamma}_m\mathbf{A}\mathbf{\Gamma}_{m'}^{\top}
\end{equation}
where  $\mathbf{A}:=\mathbf{T}\text{diag}(\bm{\pi})$. Letting $\tilde{\mathbf{S}}_{mm'}$ denote the sample counterpart of \eqref{eq:crossx_T},  and with $\{\hat{\mathbf{\Gamma}}_m \}_{m=1}^{M}$ available, we can recover $\mathbf{T}$ as follows. First, we solve the convex moment-matching optimization problem
\begin{alignat}{2}
	\label{eq:T_Opt_crosscov}
	& \underset{\mathbf{A}\in\mathcal{C}_S}{\min} && \sum_{\underset{m'>m}{m=1}}^{M}\|\tilde{\mathbf{S}}_{mm'} - \hat{\mathbf{\Gamma}}_m\mathbf{A}\hat{\mathbf{\Gamma}}_{m'}^{\top}\|_F^2 
\end{alignat}
where $\mathcal{C}_S$ is the set of matrices whose entries are positive and sum to $1$, namely $\mathcal{C}_S := \{\mathbf{X}\in\mathbb{R}^{K\times K}: \mathbf{X}\geq \bm{0}, \bm{1}^{\top}\mathbf{X}\bm{1} = 1 \}$.
The constraint is due to the fact that $\bm{1}^{\top}\mathbf{T} = \bm{1}^{\top}$, $\text{diag}(\bm{\pi})\bm{1} = \bm{\pi}$, and $\bm{\pi}^{\top}\bm{1} = 1$.
Note that \eqref{eq:T_Opt_crosscov} is a standard constrained convex optimization problem that can be solved with off-the-shelf tools, such as CVX~\cite{cvx}. Having obtained  $\hat{\mathbf{A}}$ from \eqref{eq:T_Opt_crosscov}, we can then estimate the transition matrix as
\begin{equation}
	\label{eq:T_obt}
	\hat{\mathbf{T}} = \hat{\mathbf{A}}(\text{diag}(\hat{\bm{\pi}}))^{-1}.
\end{equation}
Note here that explicit knowledge of $\bm{\pi}$ is not required, as its estimate can be recovered from $\hat{\mathbf{A}}$ as
\begin{equation*}
\hat{\bm{\pi}}^\top = \bm{1}^{\top}\hat{\mathbf{A}} = \bm{1}^{\top}\hat{\mathbf{T}}\text{diag}(\hat{\bm{\pi}}) = \bm{1}^{\top}\text{diag}(\hat{\bm{\pi}}).
\end{equation*}
The following proposition argues the consistency of the transition matrix estimates $\hat{\mathbf{T}}$.

\begin{proposition}
	\label{prop:T_est}
	Given accurate estimates of $\{\mathbf{\Gamma}_m \}$ and $\bm{\pi}$, the estimate $\hat{\mathbf{T}}$ given by \eqref{eq:T_Opt_crosscov} and \eqref{eq:T_obt} approaches $\mathbf{T}$ as $N\rightarrow\infty$.
\end{proposition}
\begin{proof}[Proof]
	By the law of large numbers, $\tilde{\mathbf{S}}_{mm'} \rightarrow \tilde{\mathbf{R}}_{mm'}$ as $N\rightarrow\infty$ for all $m,m'$. Since the objective in \eqref{eq:T_Opt_crosscov} is convex, from \cite{vapnik2013nature}, we have that $\hat{\mathbf{A}}$ will converge to $\mathbf{A} = \mathbf{T}\text{diag}(\bm{\pi})$ as $N\rightarrow \infty$. Finally, as $\hat{\mathbf{T}}$ can be recovered from $\hat{\mathbf{A}}$ in closed form [cf.~\eqref{eq:T_obt}], the proof is complete. 
\end{proof}

With estimates of $\{\hat{\mathbf{\Gamma}}_m\}, \hat{\bm{\pi}}$ and $\hat{\mathbf{T}}$ at hand, estimates of the labels $\{y_n\}_{n=1}^{N}$ can be obtained using the method described in Sec.~\ref{ssec:Viterbi}. Futhermore, the estimates of $\{\hat{\mathbf{\Gamma}}_m\}, \hat{\bm{\pi}}$ and $\hat{\mathbf{T}}$ can be used to initialize an EM algorithm (a.k.a. Baum-Welch), whose details are provided in the next subsection.

\begin{algorithm}[tb]
	\begin{algorithmic}[1]
		\algrenewcommand\algorithmicindent{1em}
		
		\Require{\textcolor{black}{Learner} responses $\{f_m(x_n)\}_{m=1,n=1}^{M,N}$, initial estimates $\mathbf{T}^{(0)}, \{\mathbf{\Gamma}_m^{(0)}  \}_{m=1}^{M}$.}
		\Ensure{Estimates $\hat{\mathbf{T}},\{ \hat{\mathbf{\Gamma}}_m\}_{m=1}^{M}$. } 
		\While{not converged}
		\State\parbox[t]{\dimexpr\linewidth-3\dimexpr\algorithmicindent}{Estimate $\hat{q}_{nk}^{(i+1)}$ and $\hat{\xi}_{n}^{(i+1)}(k,k')$ using the forward-backward algorithm (App.~\ref{app:fwdbwd}).}
		\State\parbox[t]{\dimexpr\linewidth-3\dimexpr\algorithmicindent}{Estimate $\{\hat{\mathbf{\Gamma}}_{m}^{(i+1)}\}_{m=1}^{M}$ via \eqref{eq:BW_Gamma_upd}.}
		\State\parbox[t]{\dimexpr\linewidth-3\dimexpr\algorithmicindent}{Estimate $\hat{\mathbf{T}}^{(i+1)}$ via \eqref{eq:BW_T_upd}.}
		\State\parbox[t]{\dimexpr\linewidth-3\dimexpr\algorithmicindent}{$i\leftarrow i + 1$}
		\EndWhile
	\end{algorithmic}
	\caption{EM algorithm for Sequential Data}\label{alg:BaumWelch} 
\end{algorithm}

\begin{algorithm}[tb]
	\begin{algorithmic}[1]
		\algrenewcommand\algorithmicindent{1em}
		
		\Require{\textcolor{black}{Learner} responses $\{f_m(x_n)\}_{m=1,n=1}^{M,N}.$}
		\Ensure{Estimates of data labels $\{\hat{y}_n\}_{n=1}^{N}$. } 
		
		\State\parbox[t]{\dimexpr\linewidth-3\dimexpr\algorithmicindent}{Estimate $\bm{\pi},\{\mathbf{\Gamma}_{m}\}_{m=1}^{M}$ via \eqref{eq:ConfMat_Opt_crosscov}.}
		\State\parbox[t]{\dimexpr\linewidth-3\dimexpr\algorithmicindent}{Estimate $\hat{\mathbf{T}}$ via \eqref{eq:T_Opt_crosscov} and \eqref{eq:T_obt}.}
		\State\parbox[t]{\dimexpr\linewidth-3\dimexpr\algorithmicindent}{Estimate $\hat{y}_n$  using the Viterbi algorithm [cf.~Sec.~\ref{ssec:Viterbi}].}
		\State\parbox[t]{\dimexpr\linewidth-3\dimexpr\algorithmicindent}{If needed refine estimates of $\hat{\mathbf{T}},\{\hat{\mathbf{\Gamma}}_m \}$ and $\{\hat{y}_n\}$ using Alg.~\ref{alg:BaumWelch}.}
	\end{algorithmic}
	\caption{Unsupervised Ensemble Classifier for Sequential Data}\label{alg:ensemble_classification} 
\end{algorithm}

\begin{remark}
	While here we employed the algorithm of \cite{traganitis2018} to estimate $\{\mathbf{\Gamma}_m \}$, any other unsupervised ensemble classification algorithm, such as~\cite{dawid1979maximum,zhang2014spectral}, can be utilized too. In addition, methods that jointly estimate confusion matrices and Markov chain parameters such as~\cite{xiaokejun}, can also be appropriately modified for the ensemble classification task.
\end{remark}
\subsection{EM algorithm for sequential data}
\label{ssec:BW_alg}
As with the i.i.d. case of Sec.~\ref{sec:recap_CI}, the EM algorithm developed here iteratively maximizes the log-likelihood of the observed \textcolor{black}{learner} responses. In order to update the parameters of interest $\bm{\theta} := {\rm vec}([\mathbf{T},\mathbf{\Gamma}_1,\ldots,\mathbf{\Gamma}_M])$ per iteration, the following two quantities have to be found
\begin{equation}
q_{nk} = \prob(y_n = k|\mathbf{F},\bm{\theta})
\end{equation}
and
\begin{equation}
\xi_{n}(k,k') = \prob(y_n = k, y_{n+1} = k' |\mathbf{F},\bm{\theta} )\;.
\end{equation}
Luckily, due to the causal structure of $\prob(\bm{y})$, the aforementioned quantities can be estimated efficiently using the forward-backward algorithm~\cite{hmmtutor}, whose details can be found in Appendix A of the supplementary material. 

At iteration $i$, after obtaining $q_{nk}^{(i+1)},\xi_{n}^{(i+1)}(k,k')$ for $k,k'=1,\ldots,K$ and $n=1,\ldots,N$, via the forward-backward algorithm, the transition and confusion matrix estimates can be updated as
\begin{align}
\label{eq:BW_T_upd}[\hat{\mathbf{T}}^{(i+1)}]_{k'k} & = \frac{\sum_{n=1}^{N-1}\xi_n^{(i+1)}(k',k)}{\sum_{n=1}^{N-1}q_{nk'}^{(i+1)}} \\ 
[\hat{\mathbf{\Gamma}}_m^{(i+1)}]_{k'k} & = \frac{\sum_{n=1}^{N}q_{nk}^{(i+1)}\mathcal{I}(f_m(x_n) = k')}{\sum_{k^{''}=1}^{K}\sum_{n=1}^{N}q_{nk}^{(i+1)}\mathcal{I}(f_m(x_n) = k^{''})}.\label{eq:BW_Gamma_upd} 
\end{align}
The EM iterations for sequential data with dependent labels is summarized in Alg.~\ref{alg:BaumWelch}, while the overall ensemble classifier for sequential data is tabulated in Alg.~\ref{alg:ensemble_classification}. 
\textcolor{black}{ Note that the EM algorithm of this subsection does not rely on As\ref{ass:mixing}.}

\section{Network Dependent Data}
\label{sec:gen}
{\color{black}To tackle the networked data case, this section will introduce our novel approach to unsupervised ensemble classification of networked data. Given a graph $\mathcal{G}$ encoding data dependencies, recall from Sec.~\ref{sec:prelim} that the joint pmf of all labels follows an MRF; thus, $\prob(\bm{y}) = \frac{1}{Z}\text{exp}(-U(\bm{y}))$ with 
\begin{equation*}
U(\bm{y}) = \frac{1}{2}\sum_{(n,n')\in\mathcal{E}}V(y_n,y_{n'})
\end{equation*}
where $V(y_n,y_{n'})$ is the clique potential of the $(n,n')$-th edge. 
 }
Here, we select the clique potential as
\begin{equation}
V(y_n,y_{n'}) := \begin{cases}
0 & \text{ if } y_n = y_{n'} \\
\delta_{n n'} & \text{ if } y_n \neq y_{n'}
\end{cases},
\end{equation}
where $\delta_{n n'} > 0$ is some predefined scalar, which controls how much we trust the given graph $\mathcal{G}$.
The local energy at node (datum) $n$ of the graph is then defined as
\begin{equation}
U_n(y_n) = \frac{1}{2}\sum_{n'\in\mathcal{N}_n}V(y_n,y_{n'}).
\end{equation}
This particular choice of clique potentials forces neighboring nodes (data) of the graph to have similar labels, and has been successfully used in image segmentation~\cite{MF,ICM}.
{\color{black}
Under As\ref{ass:annot_indep} and the aforementioned model, the joint pmf of label $y_n$ and corresponding \textcolor{black}{learner} responses $\{f_m(x_n)\}_{m=1}^{M}$ given the neighborhood $\bm{y}_{\mathcal{N}_n}$ of node $n$, can be expressed as
\begin{align}
& \prob\left(\{ f_m(x_n)\}_{m=1}^{M}, y_n = k | \bm{y}_{\mathcal{N}_n} = \bm{k}_{\mathcal{N}_n}\right)  \\ & =  \prod_{m=1}^{M}\Gamma_m(f_m(x_n),k)\prob(y_n = k|\bm{y}_{\mathcal{N}_n} = \bm{k}_{\mathcal{N}_n}) \notag
\end{align}
and accordingly the posterior probability of label $y_n$ as
\begin{align}
\label{eq:poster_MRF}
& \prob\left(y_n=k|\{f_m(x_n)\}_{m=1}^{M}, \bm{y}_{\mathcal{N}_n} = \bm{k}_{\mathcal{N}_n}\right) \\
& \propto  \prod_{m=1}^{M}\Gamma_m(f_m(x_n),k)\prob(y_n = k|\bm{y}_{\mathcal{N}_n} = \bm{k}_{\mathcal{N}_n}) \notag \\
& = \exp\left(-U_n(k) + \sum_{m=1}^{M}\log\Gamma_m(f_m(x_n),k)\right). \notag
\end{align}
}

\subsection{Label estimation for networked data}
\label{ssec:MRF_label_est}
Finding ML estimates of the labels $\hat{\bm{y}}$, under the aforementioned model, involves the following optimization problem
\begin{align}
\label{eq:MRF_MAP}
\hat{\bm{y}} & = \underset{\bm{y}}{\arg\max} \prob(\mathbf{F},\bm{y}) = \underset{\bm{y}}{\arg\max} \prob(\mathbf{F} | \bm{y})\prob(\bm{y}) \\
& = \underset{\bm{y}}{\arg\max}  \frac{1}{Z}\text{exp}(-U(\bm{y})) \prob(\mathbf{F} | \bm{y}). \notag
\end{align}
Unfortunately, \eqref{eq:MRF_MAP} is intractable even for relatively small $N$, due to the structure of \eqref{eq:MRF_prob}. This motivates well approximation techniques to obtain estimates of the labels. 

Popular approximation methods include Gibbs sampling~\cite{gibbs} and mean-field approximations~\cite{MF}. Here, we opted for an iterative method called iterated conditional modes (ICM), which has been used successfully in image segmentation~\cite{ICM}. Per ICM iteration, we are given estimates $\{\hat{\mathbf{\Gamma}}_m\}_{m=1}^{M}$, and update the label of datum $n$ by finding the $k$ maximizing its local posterior probability; that is,
\begin{align}
\tilde{y}_n^{(t)} & = \underset{k\in\{1,\ldots,K\}}{\arg\max}\prob\left(y_n=k|\{f_m(x_n) \}_{m=1}^{M}, \tilde{\bm{y}}_{\mathcal{N}_n}^{(t-1)}\right) \notag \\
& = \underset{k\in\{1,\ldots,K\}}{\arg\min}U_n(k) - \sum_{m=1}^{M}\log\left(\hat{\Gamma}_{m}(f_m(x_n),k)\right) \label{eq:ICM_basic}
\end{align} 
where the superscript denotes the iteration index, $\tilde{\bm{y}}_{\mathcal{N}_n}$ denotes the label estimates provided by the previous ICM iteration, and the second equality is due to \eqref{eq:poster_MRF}. The optimization in \eqref{eq:ICM_basic} is carried out for $n=1,\ldots,N$ until the values of $\tilde{\bm{y}}$ have converged or until a maximum number of iterations $T_{\rm max}$ has been reached.

The next subsection puts forth an EM algorithm for estimating $\{\hat{y}_n\}_{n=1}^{N}$ and $\{\hat{\mathbf{\Gamma}}_m \}_{m=1}^{M}$.

\subsection{EM algorithm for networked data}
\label{ssec:MRF_EM}
As with the i.i.d. case in Sec.~\ref{sec:recap_CI} and the sequential case in Sec.~\ref{sec:seq}, the EM algorithm of this section seeks to iteratively maximize the  marginal log-likelihood of observed \textcolor{black}{learner} responses. However, the Q-function [cf. Sec.~\ref{sssec:EM_iid}] is now cumbersome to compute under the MRF constraint on $\bm{y}$. 

For this reason, we resort to the approximation technique of the previous subsection to compute estimates of $q_{nk} = \prob(y_n = k | \{f_m(x_n)\}_{m=1}^{M}; \bm{\theta})$. Specifically, per EM iteration $i$, we let $\hat{\bm{y}}^{(i)} := [\hat{y}_1^{(i)},\ldots,\hat{y}_N^{(i)}]$ denote the estimates obtained by the iterative procedure of Sec.~\ref{ssec:MRF_label_est}. Then, estimates $\hat{q}_{nk}^{(i+1)}$ are obtained as [cf.~\ref{eq:poster_MRF}]
\begin{align}
\label{eq:MRF_q}
& \hat{q}_{nk}^{(i+1)} \\ & = \frac{1}{Z'}\exp\left(-U_{n}^{(i+1)}(k) + \sum_{m=1}^{M}\log\left(\hat{\Gamma}_{m}^{(i)}(f_m(x_n),k)\right)\right) \notag
\end{align}
where
\begin{equation*}
Z' = \sum_{k}\exp\left(-U_{n}^{(i+1)}(k) + \sum_{m=1}^{M}\log\left(\hat{\Gamma}_{m}^{(i)}(f_m(x_n),k)\right)\right)
\end{equation*} 
is the normalization constant, and $U_n^{(i+1)}(k)$ is given by 
\begin{equation}
U_n^{(i+1)}(k) = \frac{1}{2}\sum_{n'\in\mathcal{N}_n}V(k,\hat{y}_{n'}^{(i+1)}).
\end{equation}
Finally, the M-step that involves finding estimates of $\{\mathbf{\Gamma}_m\}_{m=1}^{M}$ is identical to the M-step of the EM algorithm of Sec.~\ref{sssec:EM_iid} for i.i.d. data; that is, 
\begin{equation}
\label{eq:MRF_gammaupd}
[\hat{\mathbf{\Gamma}}_m^{(i+1)}]_{k'k} = \frac{\sum_{n=1}^{N}\hat{q}_{nk}^{(i+1)}\mathcal{I}(f_m(x_n) = k')}{\sum_{k^{''}=1}^{K}\sum_{n=1}^{N}\hat{q}_{nk}^{(i+1)}\mathcal{I}(f_m(x_n) = k^{''})}.
\end{equation} 

\begin{algorithm}[tb]
	\begin{algorithmic}[1]
		\algrenewcommand\algorithmicindent{1em}
		
		\Require{\textcolor{black}{Learner} responses $\{f_m(x_n)\}_{m=1,n=1}^{M,N}$, initial $\bm{y}^{(0)}, \{ \mathbf{\Gamma}_m^{(0)} \}_{m=1}^{M}$, Data graph $\mathcal{G}(\mathcal{V},\mathcal{E}).$}
		\Ensure{Estimates of data labels $\{\hat{y}_n\}_{n=1}^{N}.$ } 
		
		\While{ not converged }
		\While{ not converged AND $t < T_{\rm max}$ }
		\For {$n=1,\ldots,N$}
		\State\parbox[t]{\dimexpr\linewidth-3\dimexpr\algorithmicindent}{Update $\tilde{y}_n^{(t)}$ using \eqref{eq:ICM_basic}.}
		\EndFor
		\State\parbox[t]{\dimexpr\linewidth-3\dimexpr\algorithmicindent}{$t \leftarrow t + 1$}
		\EndWhile
		\State\parbox[t]{\dimexpr\linewidth-3\dimexpr\algorithmicindent}{Compute $\hat{q}_{nk}^{(i+1)}$ using \eqref{eq:MRF_q}.}
		\State\parbox[t]{\dimexpr\linewidth-3\dimexpr\algorithmicindent}{Compute $\{ \hat{\mathbf{\Gamma}}_m^{(i+1)}\}_{m=1}^{M}$ using \eqref{eq:MRF_gammaupd}.}
		\State\parbox[t]{\dimexpr\linewidth-3\dimexpr\algorithmicindent}{$i \leftarrow i + 1$}
		\EndWhile
	\end{algorithmic}
	\caption{EM algorithm for networked data}\label{alg:MRF_EM} 
\end{algorithm}

Similar to the i.i.d. case, the aforementioned EM solver deals with a non-convex problem. In addition, the ICM method outlined in Sec.~\ref{ssec:MRF_label_est} is a deterministic approach that performs greedy optimization. Therefore, proper initialization is crucial for obtaining accurate estimates of the labels and \textcolor{black}{learner} confusion matrices. 

As with the decoupling approach of Sec.~\ref{sec:seq}, here we first obtain estimates of \textcolor{black}{learner} confusion matrices $\{\hat{\mathbf{\Gamma}}_m \}_{m=1}^{M}$ and labels $\hat{\bm{y}}$, using the moment-matching algorithm of Sec.~\ref{sssec:MM_iid}. These values are then provided as initialization to Alg.~\ref{alg:MRF_EM}.
In cases where $N$ is small to have accurate moment estimates, majority voting can be used instead to initialize Alg.~\ref{alg:MRF_EM}.
The entire procedure for unsupervised ensemble classification with networked data is tabulated in Alg. \ref{alg:ensemble_class_gendep}.

The next section will evaluate the performance of our proposed schemes.
{\begin{algorithm}[tb]
		\begin{algorithmic}[1]
			\algrenewcommand\algorithmicindent{1em}
			
			\Require{\textcolor{black}{Learner} responses $\{f_m(x_n)\}_{m=1,n=1}^{M,N}$, Data graph $\mathcal{G}(\mathcal{V},\mathcal{E})$}
			\Ensure{Estimates of data labels $\{\hat{y}_n\}_{n=1}^{N}$ } 
			
			\State\parbox[t]{\dimexpr\linewidth-3\dimexpr\algorithmicindent}{Estimate initial values of $\{\mathbf{\Gamma}_{m}\}_{m=1}^{M}$ via \eqref{eq:ConfMat_Opt_crosscov}.}
			\State\parbox[t]{\dimexpr\linewidth-3\dimexpr\algorithmicindent}{Estimate initial values of $\{ \hat{y}_n \}_{n=1}^{N}$  using \eqref{eq:label_est}.}
			\State\parbox[t]{\dimexpr\linewidth-3\dimexpr\algorithmicindent}{Refine estimates of $\{ \hat{y}_n \}_{n=1}^{N}$ and $ \{\hat{\mathbf{\Gamma}}_m  \}_{m=1}^{M}$ using Alg.~\ref{alg:MRF_EM}.}
		\end{algorithmic}
		\caption{Unsupervised Ensemble Classifier for Networked data}\label{alg:ensemble_class_gendep} 
\end{algorithm}}

{\color{black}
\begin{remark}
	Contemporary Bayesian inference tools, such as variational inference~\cite{variational}, can also be appropriately modified to estimate labels of networked data that are expected to increase classification performance.
\end{remark}
}

\section{Numerical Tests}\label{sec:numerical_tests}
The performance of the novel algorithms for both sequential and networked data is evaluated in this section using synthetic and real datasets. \textcolor{black}{To showcase the importance of accounting for data dependencies in the unsupervised ensemble task, the proposed algorithms are compared to their counterparts designed for i.i.d. data. Since most of the numerical tests involve data with multiple, and potentially imbalanced classes, unless otherwise noted, the metrics evaluated are the per-class precision, per-class recall and the overall F-score~\cite{powers2011evaluation}. F-score is the harmonic mean of precision and recall, that is }
\begin{equation}
\text{F-score} = \frac{2}{K}\sum_{k=1}^{K}\frac{\text{Precision}_k * \text{Recall}_k}{\text{Precision}_k + \text{Recall}_k}
\end{equation}
\textcolor{black}{where $\text{Precision}_k, \text{Recall}_k$ denote the per-class precision and recall, respectively. Precision$_k$ for a class $k$ measures the proportion of the data predicted to be in class $k$ that are actually from this class. Recall$_k$ for a class $k$ on the other hand measures the proportion of data that were actually in class $k$ and were predicted to be in class $k$.} \textcolor{black}{To assess how accurately the algorithms can recover learner parameters,} the average confusion matrix estimation error is also evaluated on synthetic data, as 
\begin{alignat}{2}
\bar{\varepsilon}_{CM} &:= \frac{1}{M}\sum_{m=1}^{M}\frac{\|\mathbf{\Gamma}_m - \hat{\mathbf{\Gamma}}_m\|_1}{\|\mathbf{\Gamma}_m\|_1} = \frac{1}{M}\sum_{m=1}^{M}{\|\mathbf{\Gamma}_m - \hat{\mathbf{\Gamma}}_m\|_1} 
\end{alignat}
All results represent averages over 10 independent Monte Carlo runs, using MATLAB~\cite{MATLAB:2015}. Vertical lines in some figures indicate standard deviation. 

\subsection{Sequential data}
For sequential data, Alg.~\ref{alg:ensemble_classification} with and without EM refinement (denoted as \emph{Alg.~\ref{alg:ensemble_classification} + Alg.~\ref{alg:BaumWelch}} and \emph{Alg.~\ref{alg:ensemble_classification}}, respectively) is compared to \textcolor{black}{the single best classifier, with respect to F-score, of the ensemble (denoted as \emph{Single best});} majority voting (denoted as \emph{MV}); the moment-matching method of \cite{traganitis2018} described in Sec.~\ref{sssec:MM_iid} (denoted as \emph{MM}); Alg.~\ref{alg:BaumWelch} initialized with majority voting (denoted as \emph{MV + Alg.~\ref{alg:BaumWelch}}); and, ''oracle'' classifiers. \textcolor{black}{ ``Oracle'' classifiers solve \eqref{eq:label_est_seq} using Viterbi's algorithm~\cite{viterbi}, and have access to ground-truth learner confusion matrices $\{ \mathbf{\Gamma}_m \}_{m=1}^{M}$ and the ground-truth Markov chain transition matrix $\mathbf{T}$. These ``oracle'' classifiers are used as an ideal benchmark for all other methods.} The transition matrix estimation error $\|\mathbf{T} - \hat{\mathbf{T}}\|_1$ is also evaluated using synthetic data. For real data tests, instead of \emph{MM} the EM algorithm of~\cite{dawid1979maximum} initialized with \emph{MM} is evaluated (denoted as \emph{DS}.)

All datasets in this subsection are split into sequences. Here, we assume that per dataset these sequences are drawn from the same ensemble HMM [cf. Sec.~\ref{sec:seq}]. The reported F-score represents the averaged F-score from all sequences.
\subsubsection{Synthetic data}
For synthetic tests,  $S$ sequences of $N_s, s=1,\ldots S$, ground-truth labels each, were generated from a Markov chain, whose transition matrix was drawn at random such that $\mathbf{T}\in\mathcal{C}$.  Each of the $N = \sum_{s}N_s$ ground-truth labels $\{y_n\}_{n=1}^{N}$ corresponds to one out of $K$ possible classes. Afterwards, $\{\mathbf{\Gamma}_m\}_{m=1}^{M}$ were generated at random, such that  $\mathbf{\Gamma}_m\in\mathcal{C}$, for all $m=1,\ldots,M$, and $\lfloor M/2\rfloor + 1$ \textcolor{black}{learner}s are better than random, as per As\ref{ass:better_than_random}. Then \textcolor{black}{learner}s' responses were generated as follows: if $y_n = k$, then the response of \textcolor{black}{learner} $m$ will be generated randomly according to the $k$-th column of its confusion matrix, $\bm{\gamma}_{m,k}$ [cf. Sec.~\ref{sec:prelim}], that is $f_m(x_n) \sim \bm{\gamma}_{m,k}$. 

Fig.~\ref{fig:seq_N_F1} shows the average F-score for a synthetic dataset with $K=4$, $M=10$ \textcolor{black}{learner}s and a variable number of data $N$ and $N_s = 40$ for all $s=1,\ldots,S$. Fig.~\ref{fig:seq_N} shows the average confusion and transition matrix estimation errors for varying $N$. As the number of data $N$ increases the performance of the proposed methods approaches the performance of the ``oracle'' one. Accordingly, the confusion and transition matrix estimates are approaching the true ones as $N$ increases. This is to be expected, as noted in \cite{traganitis2018}, since the estimated moments are more accurate for large $N$. Interestingly, Alg.~\ref{alg:BaumWelch} performs well when initialized with majority voting, even though it reaches a performance plateau as $N$ increases. For small $N$ however, it outperforms the other proposed methods. This suggests that initializing 
Alg.~\ref{alg:BaumWelch} with majority voting is preferable when $N$ is not large enough to obtain accurate moment estimation.

The next experiment evaluates the influence of the number of \textcolor{black}{learner}s $M$ for the sequential classification task. Figs.~\ref{fig:seq_M_F1} and~\ref{fig:seq_M} showcase results for an experiment with $K=4$, fixed number of data $N=10^3, N_s = 40$ and a varying number of \textcolor{black}{learner}s $M$. Clearly, the presence of multiple \textcolor{black}{learner}s is beneficial, as the F-score increases for all algorithms, while the confusion and transition matrix errors decrease. As with the previous experiment, the performance of Alg.~\ref{alg:BaumWelch} + Alg.~\ref{alg:ensemble_classification} improves in terms of F-score, as $M$ increases. 
\begin{figure}[tb]
	\centering
	\includegraphics[width=\columnwidth]{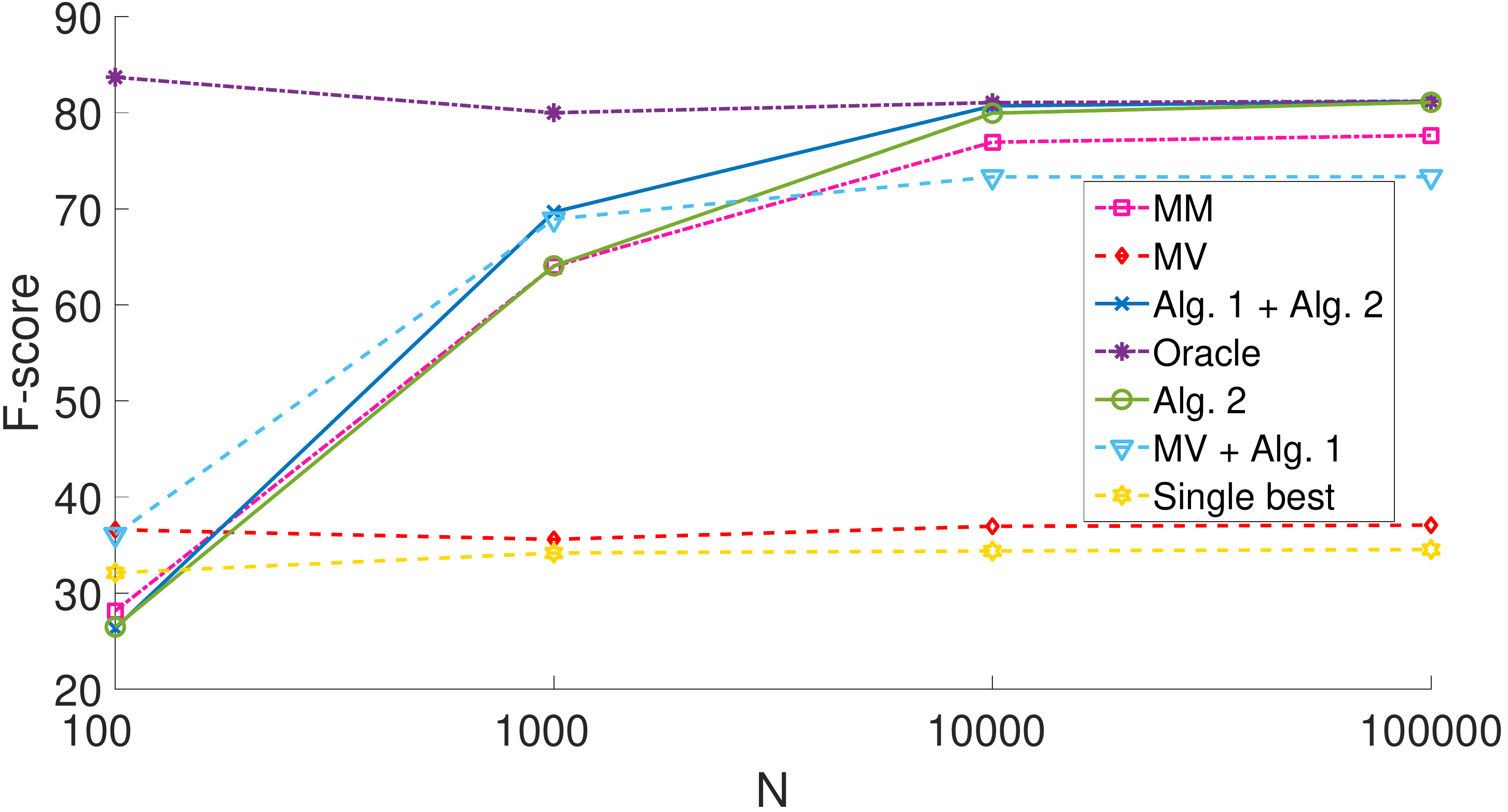}
	\caption{Average F-score for a synthetic sequential dataset with $K=4$ and $M=10$ \textcolor{black}{learner}s}
	\label{fig:seq_N_F1}
\end{figure}

\begin{figure}[tb]
	\centering
	\begin{subfloat}[Confusion matrix estimation error]{\includegraphics[width=\columnwidth]{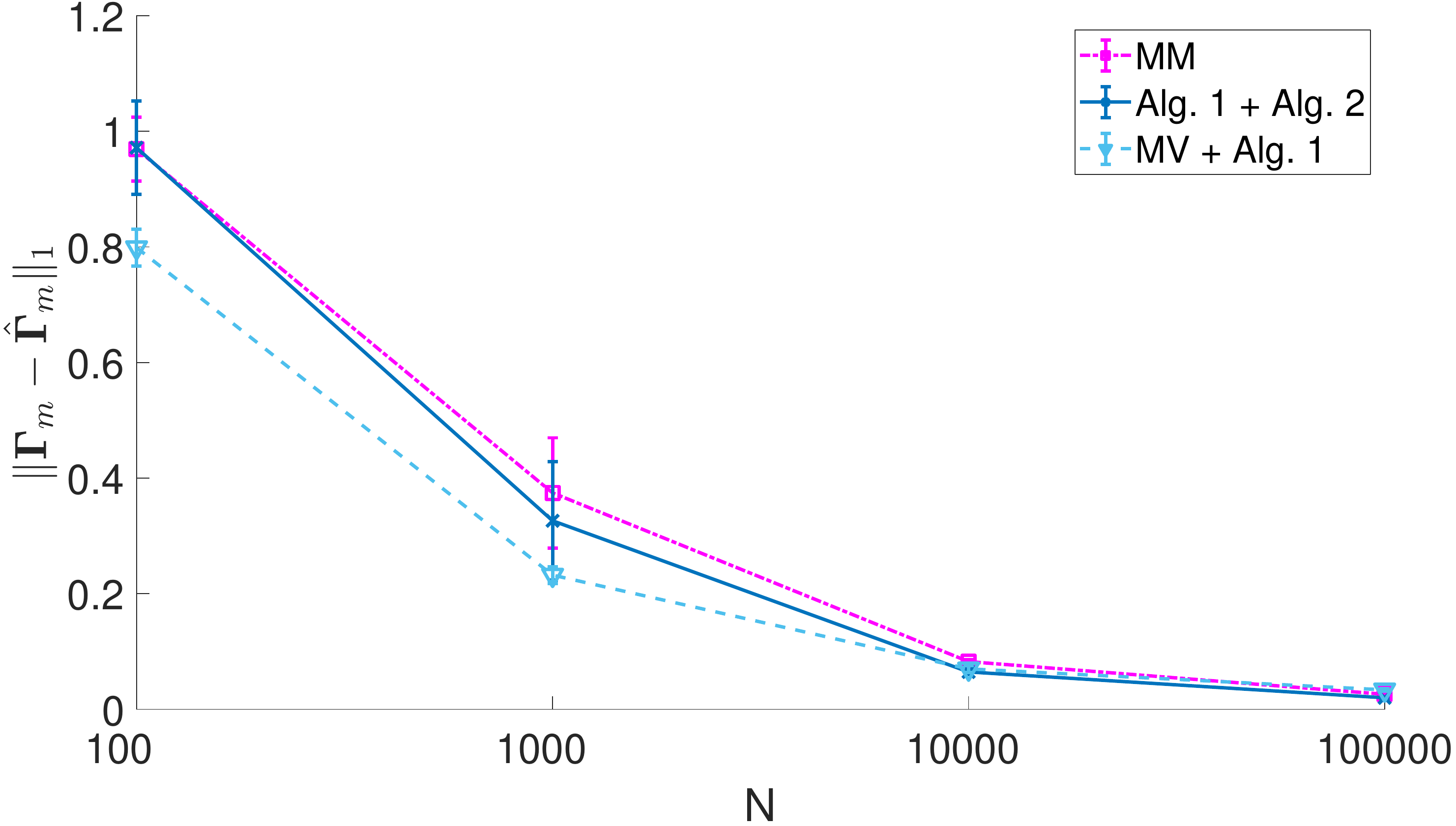} 
			\label{fig:seq_N_Gamma}}
	\end{subfloat}
	
	\begin{subfloat}[Transition matrix estimation error]{\includegraphics[width=\columnwidth]{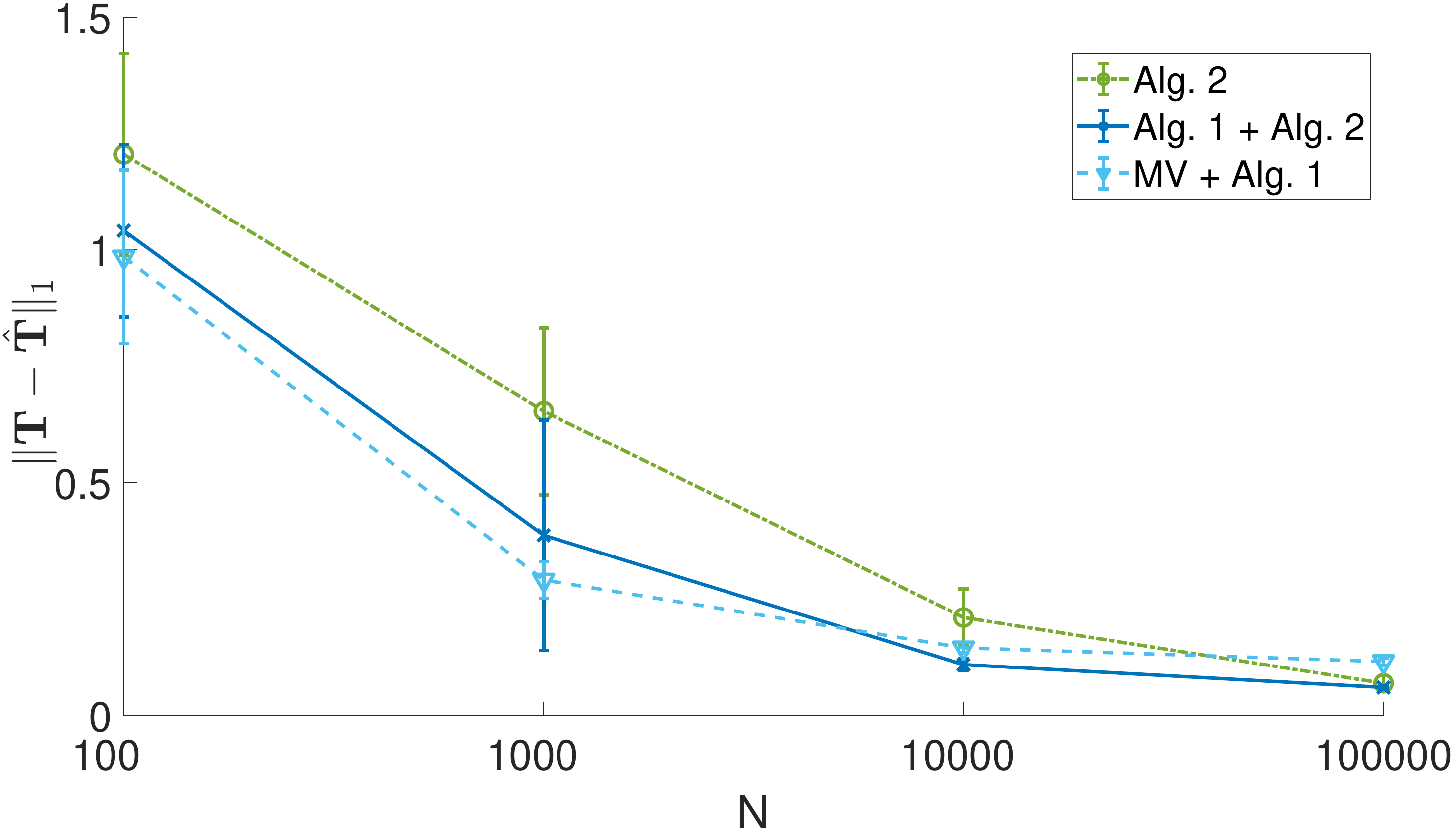}
			\label{fig:seq_N_T}}
	\end{subfloat}
	
	\caption{Average estimation errors of confusion matrices and prior probabilities for a synthetic sequential dataset with $K=4$ and $M=10$ \textcolor{black}{learner}s}\label{fig:seq_N}
\end{figure}

\begin{figure}[tb]
	\centering
	\includegraphics[width=\columnwidth]{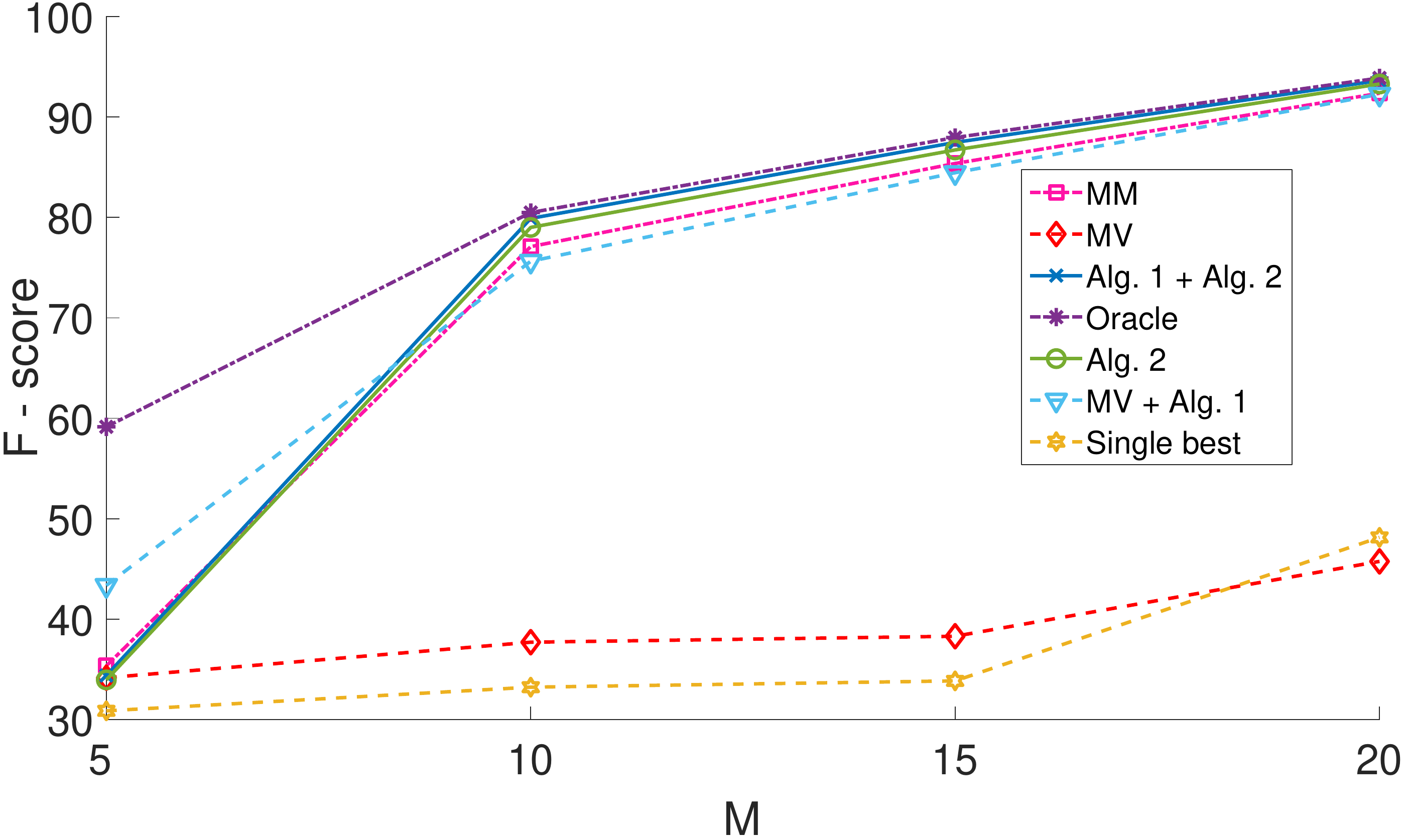}
	\caption{Average F-score for a synthetic sequential dataset with $K=4$ and $N=10^3$ data.}
	\label{fig:seq_M_F1}
\end{figure}

\begin{figure}[tb]
	\centering
	\begin{subfloat}[Confusion matrix estimation error]{\includegraphics[width=\columnwidth]{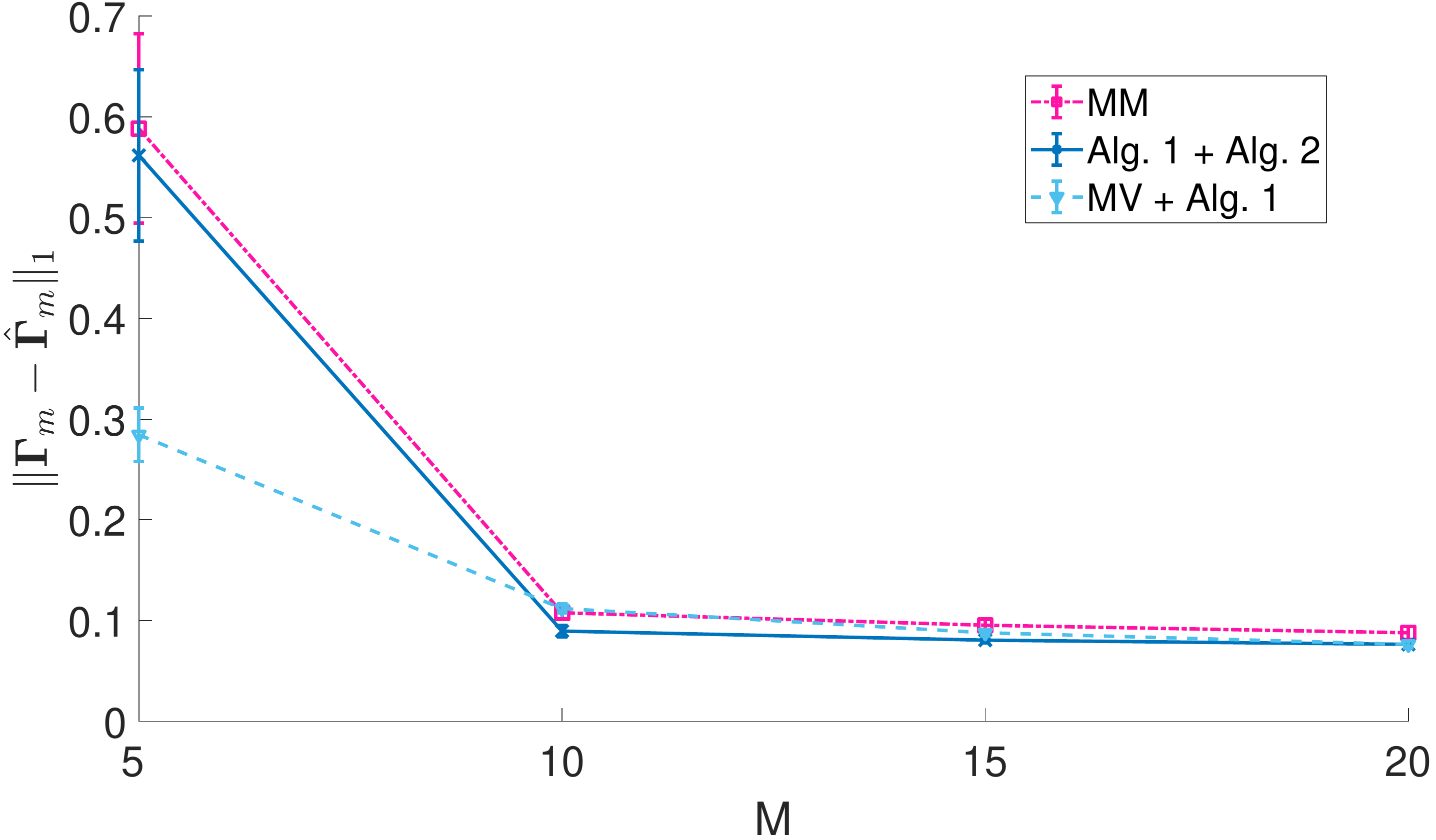} 
			\label{fig:seq_M_Gamma}}
	\end{subfloat}
	
	\begin{subfloat}[Transition matrix estimation error]{\includegraphics[width=\columnwidth]{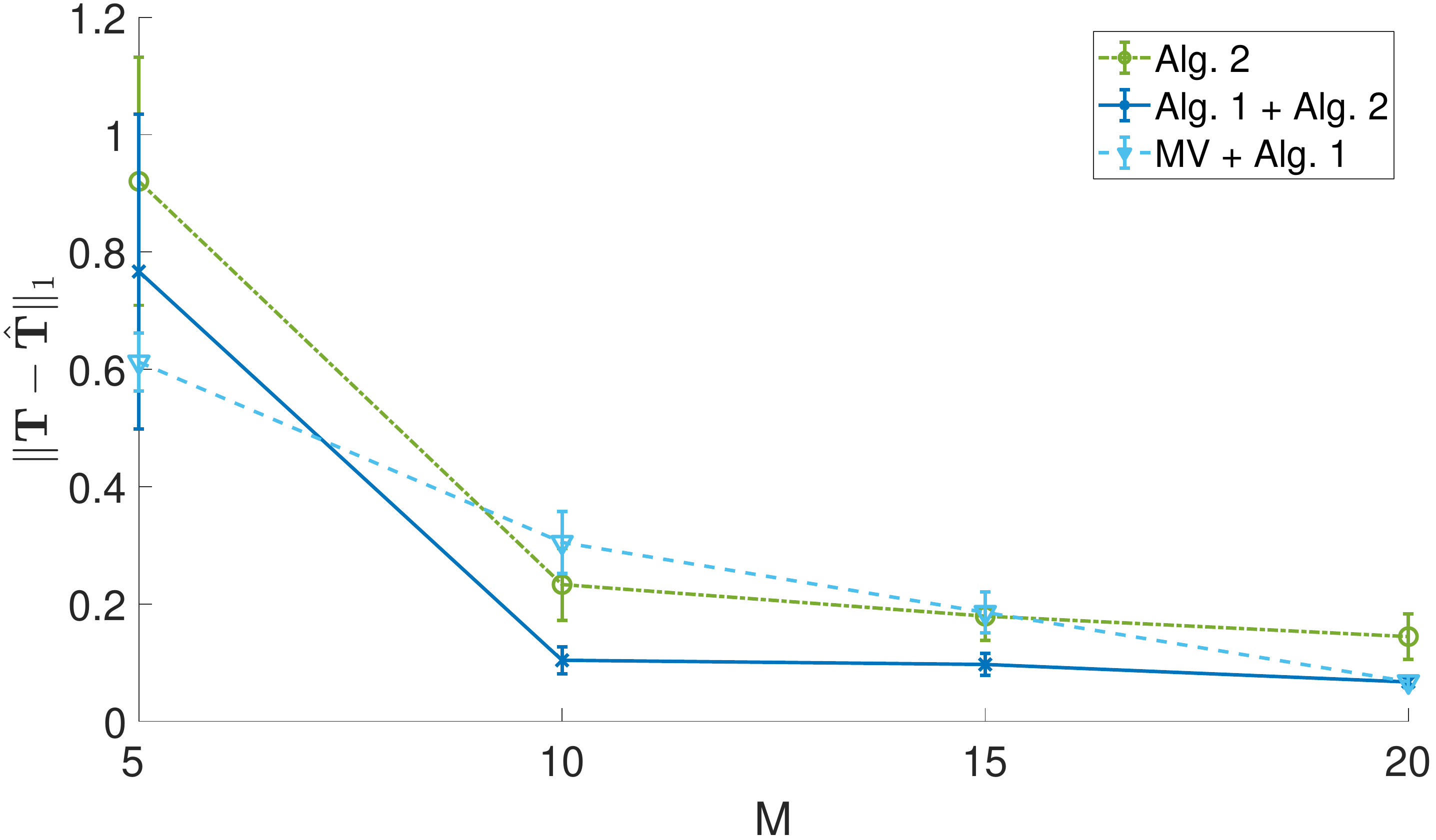}
			\label{fig:seq_M_T}}
	\end{subfloat}
	
	\caption{Average estimation errors of confusion matrices and prior probabilities for a synthetic sequential dataset with $K=2$ and $M=10$ \textcolor{black}{learner}s}\label{fig:seq_M}
\end{figure}

\subsubsection{Real data} Further tests were conducted on {three} real datasets, the part-of-speech (POS) tagging dataset, {the named entity recognition (NER) dataset}, and the biomedical information extraction (IE)~\cite{pmid_hmm_em} dataset. 

For the POS dataset $M=10$ classifiers were trained using NLTK~\cite{nltk} on subsets of the Brown coprus~\cite{browncorpus} to provide part-of-speech (POS) tags of text. The number of tags is $K=12$. Then the classifiers provided POS tags for all words in the Penn Treebank corpus~\cite{penntreebank}, which contains $N=100,676$ words. 

{The NER dataset consists of $5,916$ crowd annotated sentences from the CoNLL database~\cite{conll,Rodrigues2014}. 
The dataset contains $N= 78,107$ words belonging to $K=9$ distinct classes, each describing a different named-entity such as person, location, and  organization. Corresponding to learners in our models, $M=47$ human annotators provided labels for the words in the dataset.}

The Biomedical IE dataset consists of $5,000$ medical paper Abstracts, on which $M=312$ \textcolor{black}{human} annotators, were tasked with marking all text spans in a given Abstract that identify the population of a randomized controlled trial. The dataset consists of $N=7,880,254$ words belonging into $K=2$ classes: in a span identifying the population or outside. For this particular dataset we evaluate Precision and Recall per sequence in the following way, which was suggested in~\cite{pmid_hmm_em}
\begin{align*}
\text{Precision} & = \frac{\text{\# true positive words}}{\text{\# words in a predicted span}} \\
\text{Recall} & = \frac{\text{\# words in a predicted span}}{\text{\# words in ground-truth span}}.
\end{align*}
\textcolor{black}{These new definitions of precision and recall, allow us to credit partial matches. The justification for using these alternative definitions is that the previous ones are too strict for this task, where annotated sequences are especially long.} {We used the $M=120$ annotators that had provided the largest number of responses to mantain reasonable computational complexity.}
Results for all datasets are listed in Tab.~\ref{tab:seq_real}. For the POS dataset, it can be seen that Alg.~\ref{alg:BaumWelch} + Alg.~\ref{alg:ensemble_classification} performs best in all metrics. { Similar results are showcased for the NER dataset, with majority voting achieving the best precision, and Alg.~\ref{alg:BaumWelch} + Alg.~\ref{alg:ensemble_classification} exhibiting the best recall and overall F-score.} For the Biomedical IE dataset, while majority voting achieves the best precision of all algorithms, due to its low recall, the overall F-score is low. However, Alg.~\ref{alg:BaumWelch} + Alg.~\ref{alg:ensemble_classification} outperforms competing alternatives with regards to F-score, {while MV + Alg.~\ref{alg:BaumWelch} exhibits the best recall. Note that the single best learners for the NER and Biomedical IE datasets are evaluated only on the subsets of data for which they have provided responses; the best learner for the NER dataset annotated approximately $12,500$ words, whereas the best learner for the Biomedical IE dataset annotated approximately $14,300$ words.} { The performance of the proposed label aggregation methods relies on multiple parameters such as the number and ability of learners, and how well the proposed model approximates the learner behavior and data. The modest performance gains of Algs. \ref{alg:ensemble_classification}, \ref{alg:BaumWelch} and MV + \ref{alg:BaumWelch} compared to the single best learner may be attributed to such issues. }

\begin{table*}[tb]
	\centering
	\begin{tabular}{ | c | c | c | c || c | c | c | c | c | c | c |}
		\hline
		Dataset & K & M & N & Metric & \textcolor{black}{ Single best} & MV & DS & Alg.~\ref{alg:ensemble_classification} & Alg.~\ref{alg:BaumWelch} + Alg.~\ref{alg:ensemble_classification} & MV + Alg.~\ref{alg:BaumWelch}   \\ [0.5ex] \hline
		\multirow{3}{*}{POS} & \multirow{3}{*}{$12$} & \multirow{3}{*}{$10$} & \multirow{3}{*}{$100,676$} &
		  Precision & \textcolor{black}{$0.2316$} & $0.22916$ & $0.23406$ & $0.24785$ & $\bf0.25856$ & $0.22972$ \\\cline{5-11}
		 & & & & Recall & \textcolor{black}{$0.2573$} & $0.23518$ & $0.25629$ & $0.24396$ & $\bf0.26638$ & $0.24497$ \\\cline{5-11}
		 & & & & F-score & \textcolor{black}{$0.2356$} & $0.22598$ & $0.22264$ & $0.23352$ & $\bf0.24735$ &  $0.23007$ \\ [0.5ex] \hline
		 \multirow{3}{*}{ NER} & \multirow{3}{*}{$9$} & \multirow{3}{*}{$47$} & \multirow{3}{*}{$78,107$} &
		 Precision & { $0.9*$ } & $\bf0.79$ & $0.77$ & $0.74$ & $0.77$ & $0.75$\\ \cline{5-11}
		 & & & & Recall & {$0.24*$ } & $0.59$ & $0.66$ & $0.89$ & $\bf0.69$ & $0.66$ \\ \cline{5-11}
		 & & & & F-score & {$0.89*$ } & $0.68$ & $0.71$ & $0.62$ & $\bf0.72$ &  $0.70$ \\ \hline
		\multirow{3}{*}{Biomedical IE} & \multirow{3}{*}{$2$} & \multirow{3}{*}{$120$} & \multirow{3}{*}{$7,880,254$} &
		Precision & {$0.94*$} & {$\bf0.89$} & {$0.81$} & {$0.75$} & {$0.69$} & {$0.62$}\\ \cline{5-11}
		& & & & Recall & {$0.76*$} & {$0.45$} & {$0.57$} & {$0.60$} & {$0.68$} & {$\bf0.74$} \\ \cline{5-11}
		& & & & F-score & {$0.84*$} & {$0.6$} & {$0.66$} & {$0.67$} & {$\bf0.68$} &  {$0.67$} \\ \hline
	\end{tabular}
	\bigskip
	\caption{Results for real data experiments with sequential data. {The asterisk $*$ indicates that results are from a subset of available data.}}
	\label{tab:seq_real}
\end{table*}

\subsection{Networked data}
For networked data, Alg.~\ref{alg:ensemble_class_gendep} (denoted as \emph{Alg.~\ref{alg:ensemble_class_gendep}}) is compared to \textcolor{black}{the single best classifier, with respect to F-score, of the ensemble (denoted as \emph{Single best}),} majority voting (denoted as \emph{MV}), the moment-matching method of \cite{traganitis2018} described in Sec.~\ref{sssec:MM_iid} (denoted as \emph{MM}) and Alg.~\ref{alg:MRF_EM} initialized with majority voting (denoted as \emph{MV + Alg.~\ref{alg:MRF_EM}}). For real data tests, instead of \emph{MM} the EM algorithm of~\cite{dawid1979maximum} initialized with \emph{MM} is evaluated (denoted as \emph{DS}.)
{\color{black} The average degree $\bar{d}$ of the network is used to quantify the degree of data dependency ($\bar d $ is the number of connections averaged across nodes).} \subsubsection{Synthetic data}
For the synthetic data tests, an $N$-node, $K$ community graph is generated using a stochastic block model~\cite{sbm}. Each community corresponds to a class, and the labels $\{ y_n \}_{n=1}^{N}$ indicate the community each node belongs to, i.e. $y_n = k$ if node $n$ belongs to the $k$-th community. Afterwards, $\{\mathbf{\Gamma}_m\}_{m=1}^{M}$ were generated at random, such that  $\mathbf{\Gamma}_m\in\mathcal{C}$, for all $m=1,\ldots,M$, and \textcolor{black}{learners} are better than random, as per As\ref{ass:better_than_random}. Then \textcolor{black}{learners}' responses were generated as follows: if $y_n = k$, then the response of \textcolor{black}{learner} $m$ will be generated randomly according to the $k$-th column of its confusion matrix, $\bm{\gamma}_{m,k}$ [cf. Sec.~\ref{sec:prelim}], that is $f_m(x_n) \sim \bm{\gamma}_{m,k}$. For the synthetic data tests, we set $\delta_{n n'} = M$. Fig.~\ref{fig:mrf_N_F1} shows the average F-score for a synthetic dataset with $K=4$ and $M=10$ \textcolor{black}{learners} for varying number of data $N$. \textcolor{black}{Here the average degree is $\bar{d} = 0.5$.} Fig.~\ref{fig:mrf_N_Gamma} shows the average confusion estimation error as $N$ increases. 
 As with sequential data, the F-score of the proposed algorithms increases with $N$ growing, and confusion matrix estimation error decreases. \emph{MV+ Alg.~\ref{alg:MRF_EM}} quickly reaches a plateau of performance as \emph{MV} also does not improve with increasing $N$. At the same time \emph{Alg.~\ref{alg:ensemble_class_gendep}} capitalizes on the initialization provided by \emph{MM}.
 \textcolor{black}{
  Fig.~\ref{fig:mrf_N_F1_2} shows the F-score for a similar experiment, but with network average degree $\bar{d} = 5$, i.e. higher graph connectivity. Here, we observe algorithmic behavior similar to that of the previous experiment; however, due to the higher connectivity of the graph \emph{Alg.~\ref{alg:ensemble_class_gendep}} has a greater F-score gap to \emph{MM}. This indicates that networked data with higher connectivity benefit more from \emph{Alg.~\ref{alg:ensemble_class_gendep}} and \emph{MV+ Alg.~\ref{alg:MRF_EM}}.}

 \begin{figure}[tb]
 	\centering
 	\includegraphics[width=\columnwidth]{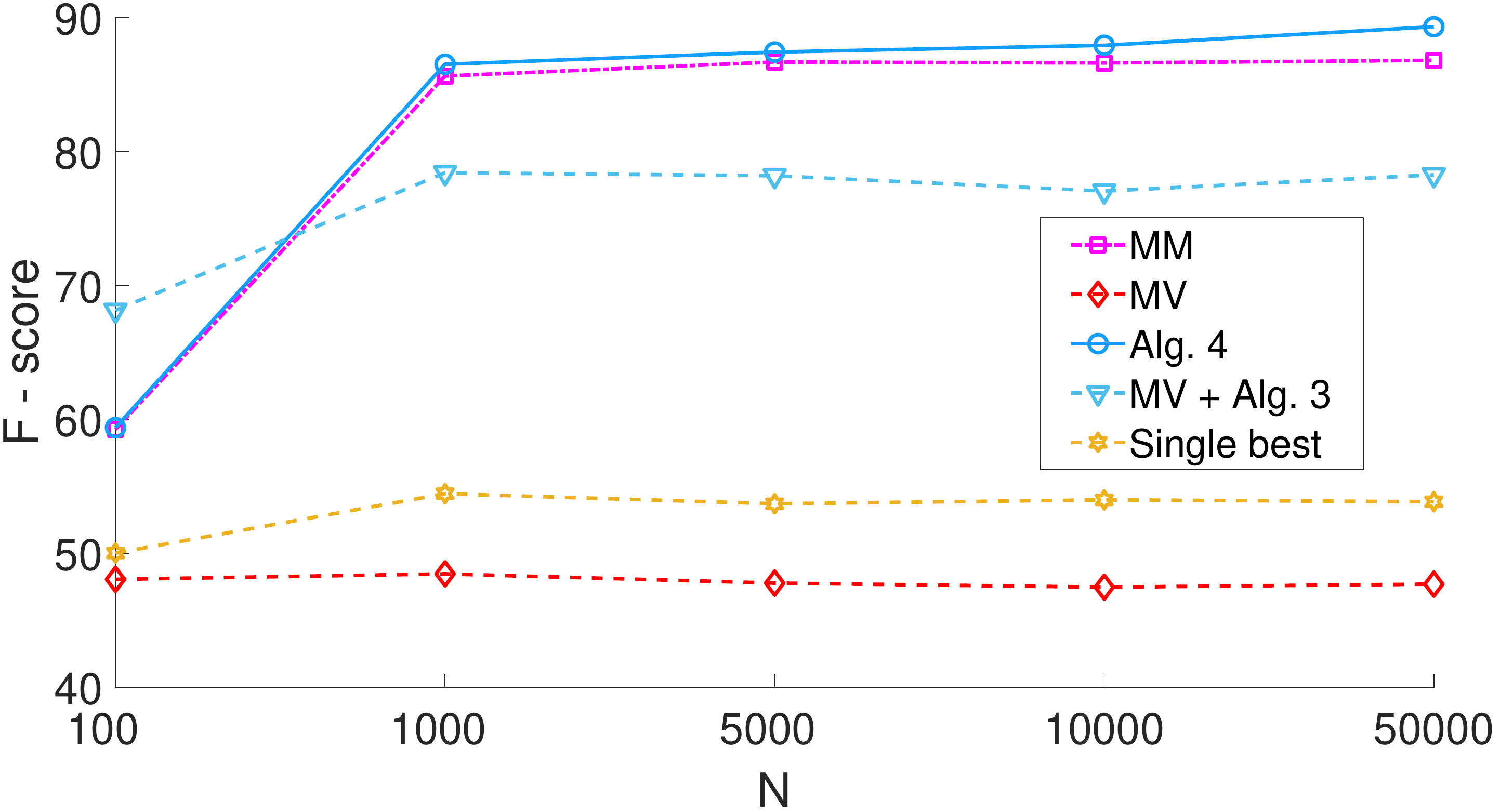}
 	\caption{Average F-score for a synthetic networked dataset with $K=4$, $M=10$ \textcolor{black}{learners} and average degree $\bar{d} = 0.5$.}
 	\label{fig:mrf_N_F1}
 \end{figure}
 
 \begin{figure}[tb]
 	\centering
 	\includegraphics[width=\columnwidth]{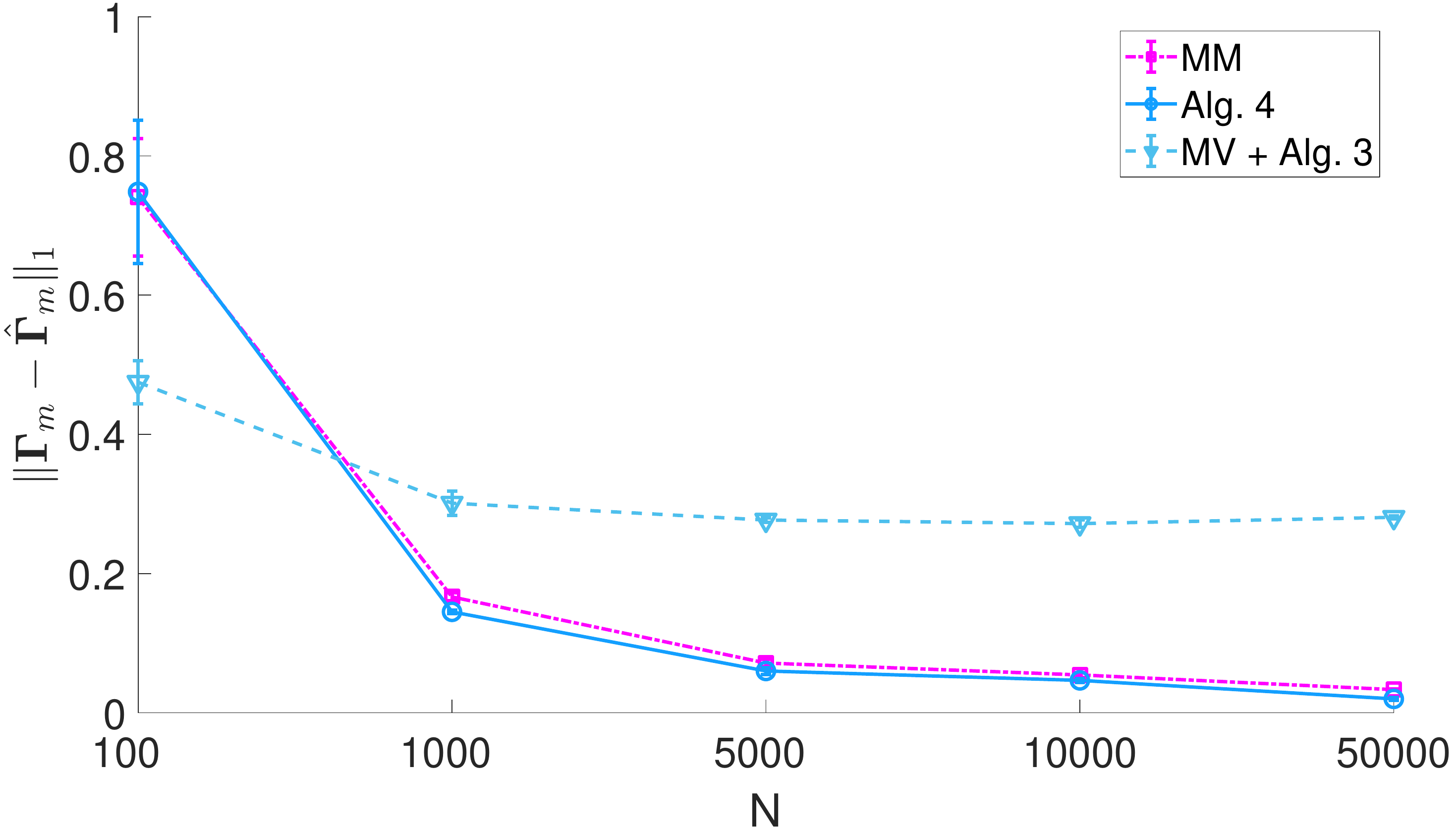} 
 	\caption{Average estimation errors of confusion matrices for a synthetic networked dataset with $K=4$ and $M=10$ \textcolor{black}{learners} }\label{fig:mrf_N_Gamma}
 \end{figure}

 \begin{figure}[tb]
	\centering
	\includegraphics[width=\columnwidth]{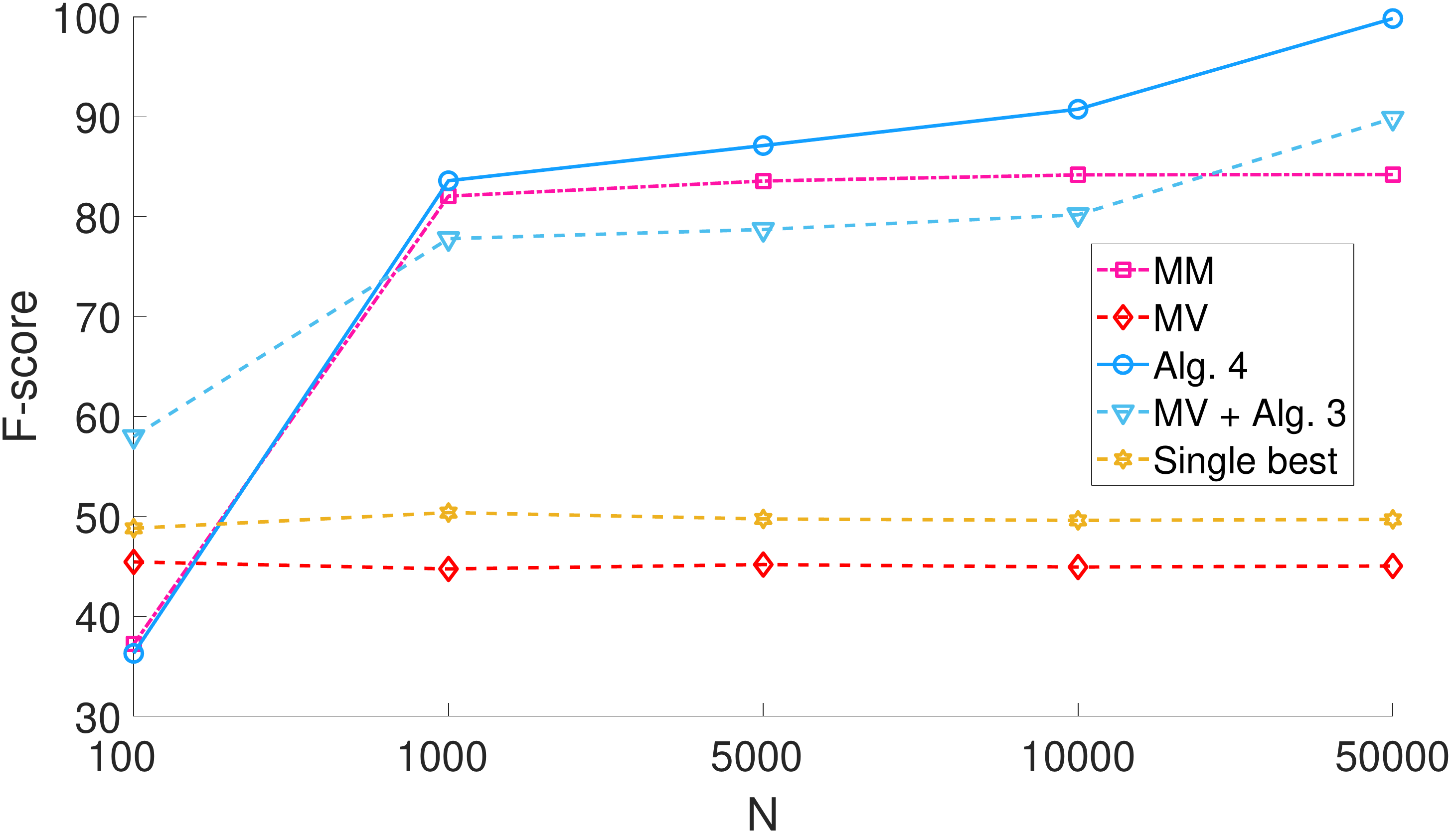}
	\caption{Average F-score for a synthetic networked dataset with $K=4$, $M=10$ \textcolor{black}{learners} and average degree $\bar{d} = 5$.}
	\label{fig:mrf_N_F1_2}
\end{figure}
 
\subsubsection{Real data}
Further tests were conducted on {six} real datasets. For the Cora, Citeseer~\cite{cora_citeseer} and Pubmed~\cite{pubmed} datasets the graph $\mathcal{G}$ and data features $\{x_n\}$ are provided with the dataset. In these cases, $M=10$ classification algorithms from MATLAB's machine learning toolbox were trained on different randomly selected subsets of the datasets. Afterwards, these algorithms provided labels for all data in the dataset.
{Boston University's biomedical image library (BU-BIL)~\cite{BU-BIL} magnetic resonance imaging (MRI) dataset contains $35$ images of rabbit aortas with size $25\times24$.  Learner responses per pixel are gathered through Amazon's mechanical turk. Here, the per-image graph $\mathcal{G}$ is the grid graph defined by the pixels of each image, that is each pixel is connected to its adjacent $8$ pixels.} For these datasets, we set $\delta_{n n'} = M$.
\textcolor{black}{The Music genre and Sentence Polarity datasets~\cite{musicgenre_senpoldata} are crowdsourcing datasets, where} the features $\{x_n\}$ are provided \textcolor{black}{and learner responses $\mathbf{F}$ are gathered through crowdsourcing platforms}. In these cases, the graphs were generated from the data features using $k$-nearest neighbors. Since the graphs are generated from the data features, here we set $\delta_{n n'} = M_n/2$, where $M_n$ is the number of \textcolor{black}{learner}s that have provided a response for the $n$-th datum. 

The Cora, CiteSeer and Pubmed datasets are citation networks and the versions used here are preprocessed by~\cite{adadif}. The Cora dataset consists of $N=2,708$ scientific publications classified into $K=7$ classes. The features $\{x_n\}$ of this dataset are sparse $1,433$-dimensional vectors and for this dataset each classification algorithm was trained on a random subset of $150$ instances.  The CiteSeer dataset consists of $N=3,312$ scientific publications classified into one of $K=6$ classes. The features $\{x_n\}$ of this dataset are sparse $3,703$-dimensional vectors, and each classification algorithm was trained on a subset of $100$ instances.
The Pubmed dataset is a citation network that consists of $N=19,717$ scientific publications from the Pubmed database pertaining to diabetes, classified into one of $K=3$ classes. The features $\{x_n\}$ of this dataset are $500$-dimensional vectors, and each classification algorithm was trained on a subset of $300$ instances. {Targeting segmentation, pixels of each image in the  BU-BIL are classified in $K=2$ classes, indicating whether they belong to a biological structure or not, by $M=7$ human annotators. The total number of pixels from the $35$ images is $N=28,181$.}
 The Music genre dataset contains $N=700$ song samples (each of duration $30$secs), belonging into $K=10$ music categories, annotated by $M=44$ \textcolor{black}{human} annotators. The sentence polarity dataset contains $N=5,000$ sentences from movie reviews, classified into $K=2$ categories (positive or negative), annotated by $M=203$ \textcolor{black}{human} annotators.

In most datasets \emph{Alg.~\ref{alg:ensemble_class_gendep}} exhibits the best performance in terms of F-score followed closely by \emph{MV+Alg.~\ref{alg:MRF_EM}}. For the Music Genre { and BU-BIL MRI} datasets however, \emph{MV+Alg.~\ref{alg:MRF_EM}} outperforms \emph{Alg.~\ref{alg:ensemble_class_gendep}}. This is to be expected for the Music Genre dataset, as $N$ is relatively small for this dataset and as such the estimated \textcolor{black}{learner} moments are not very accurate. This can also be seen from the fact that \emph{MV} outperforms \emph{DS}. {Similarly, for the BU-BIL MRI dataset \emph{MV} outperforms \emph{DS}, explaining the better performance of \emph{MV+Alg.~\ref{alg:MRF_EM}.}}
\textcolor{black}{For all datasets, \emph{Alg.~\ref{alg:ensemble_class_gendep}} and \emph{MV+Alg.~\ref{alg:MRF_EM}} consistently outperform the single best classifier. Also, note that the single best \textcolor{black}{learner}s for the Music genre and sentence polarity datasets are evaluated only on the subsets of data for which they have provided responses. In particular, the best \textcolor{black}{learner} for the Music-genre  has annotated $10$ data, while the best \textcolor{black}{learner} for the Sentence-polarity  has annotated only $6$ data.
Another interesting observation is that for most datasets having a relatively large average degree $\bar{d}$, \emph{Alg.~\ref{alg:ensemble_class_gendep}} and \emph{MV+Alg.~\ref{alg:MRF_EM}} have a greater performance gap to their counterparts that do not account for the structure of networked data. Similar to synthetic data, this suggests that well connected datasets can benefit more from these types of approaches.} { Similarly to the sequential case, the modest performance gains of Algs. \ref{alg:MRF_EM},  and MV + \ref{alg:MRF_EM} compared to the single best learner in the BU-BIL MRI dataset may be attributed to modeling discrepancies or significantly different ability levels between the learners. }
 {All in all, these results show that inclusion of graph information can be beneficial for the unsupervised ensemble or crowdsourced classification task, even when the graph is noisy (as with Sentence Polarity or the Music genre datasets.)} 

\begin{table*}[tb]
	\centering
	\begin{tabular}{|c |c | c |c | c || c  |  c |  c | c | c | } 
		\hline
		Dataset  & K & M  & N & \textcolor{black}{$\bar{d}$} & \textcolor{black}{ Single best} & MV & DS & Alg.~\ref{alg:ensemble_class_gendep} & MV + Alg.~\ref{alg:MRF_EM}  \\ [0.5ex] 
		\hline\hline
		Cora  & $10$ & $10$ & $2,708$ & \textcolor{black}{$3.9$} & \textcolor{black}{$0.51$} & $0.2785$ & $0.4228$ & $\bf0.6412$ & $0.336$ \\  \hline
		CiteSeer  & $7$ & $10$  & $3,312$ & \textcolor{black}{$2.77$} & \textcolor{black}{$0.45$} & $0.4257$ & $0.4449$ & $\bf0.5244$ & $0.5224$  \\ \hline
		Pubmed  & $3$ & $10$  & $19,717$ & \textcolor{black}{$4.49$} & \textcolor{black}{$0.717$} & $0.6968$  & $0.7437$ & $\bf0.7667$ & $0.7595$ \\ \hline
		Music Genre  & $10$ & $44$ & $700$ & \textcolor{black}{$4.8$} & {$1*$} & $0.7046$ & $0.4746$ & $0.7649$ & $\bf0.8029$\\ \hline
		Sen. Polarity  & $2$ & $203$  & $5,000$ & \textcolor{black}{$1.8$} & {$1*$} & $0.8895$ & $0.9129$ & $\bf0.9153$ & $0.9139$  \\ \hline
		{BU-BIL MRI}  & $2$ & $7$  & $28,181$ & {$7.1$} & {$0.851$} & {$0.861$} & {$0.859$} & {$0.86$} & {$\bf0.863$}  \\ \hline
	\end{tabular}
	\bigskip
	\caption{F-score for Real data experiments with Networked data. {The asterisk $*$ indicates that results are from a subset of available data.}} 	\label{tab:realdata}
\end{table*}

\section{Conclusions and future directions}\label{sec:conclusion}
This paper introduced two novel approaches to unsupervised ensemble and crowdsourced classification in the presence of data dependencies.  Two types of data dependencies were investigated: i) Sequential data; and ii) Networked data, where the dependencies are captured by a known graph. The performance of our novel schemes was evaluated on real and synthetic data.

Several interesting research avenues open up: i) Distributed and online implementations of the proposed algorithms; \textcolor{black}{ii) use of contemporary tools such as variational inference to boost performance of the novel approaches;} iii) ensemble classification with dependent classifiers and dependent data; iv) development of more realistic \textcolor{black}{learner} models for dependent data; 
v) extension of the proposed methods to semi-supervised ensemble learning; \textcolor{black}{vi) rigorous performance analysis of the proposed models}.
\bibliographystyle{IEEEtran}
\bibliography{./bib/EnsembleJournal}

\clearpage
\appendices
\section{The forward-backward algorithm}
\label{app:fwdbwd}
Let $b_{n,k}$ denote the probability of observing $\{f_m(x_n) \}_{m=1}^{M}$ given that $y_n = k$, that is 
\begin{equation}
\label{eq:bobs}
b_{n,k} = \prod_{m=1}^{M}\prob(f_m(x_n)|y_n = k) = \prod_{m=1}^{M}\Gamma_m(f_m(x_n),k).
\end{equation}

The forward-backward algorithm~\cite{hmmtutor} seeks to efficiently obtain the probability of the observed variable sequence $\{f_m(x_n) \}_{n=1,m=1}^{N,M}$, given current HMM parameter estimates $\bm{\theta}$. It takes advantage of the fact that the past and future states of a Markov chain are independent given the current state. Tailored for our ensemble HMM, we have
\begin{align}
\prob(\mathbf{F}|\bm{\theta}) & = \sum_{k=1}^{K}\prob(\mathbf{F}_{1:n},y_n=k;\bm{\theta})\prob(\mathbf{F}_{n+1:N}|y_n = k;\bm{\theta}), 
\end{align}
where $\mathbf{F}_{1:n}$ is a matrix collecting all \textcolor{black}{learner} responses for $n'=1,\ldots,n$, and $\mathbf{F}_{n+1:N}$ is a matrix collecting \textcolor{black}{learner} responses for $n'=n+1,\ldots,N$.

The forward backward algorithm computes the probability of the observed sequence iteratively using so-called forward and backward variables.
Define the forward variable as
\begin{equation}
\alpha_{n,k} = \prob(\mathbf{F}_{1:n},y_n=k;\bm{\theta})
\end{equation}
Then let
\begin{equation}
\alpha_{1,k} = \prob(y_1 =k)b_{1,k} \quad \text{for}\quad k=1,\ldots,K.
\end{equation}
and for $n=1,\ldots,N$
\begin{equation}
\label{eq:fwd_var_rec}
\alpha_{n+1,k} = b_{n+1,k}\sum_{k'=1}^{K}\alpha_{n,k'}T(k,k') \;.
\end{equation}

Upon defining the backward variables as
\begin{equation}
\beta_{n,k} = \prob(\mathbf{F}_{n+1:N}|y_n = k;\bm{\theta})
\end{equation}
\begin{equation}
\beta_{N,k} = 1 \quad \text{ for } k=1,\ldots,K
\end{equation}
it holds for $n=N-1,\ldots,1$ that 
\begin{equation}
\label{eq:bwd_var_rec}
\beta_{n,k} = \sum_{k'=1}^{K}T(k,k')\beta_{n+1,k'}b_{n+1,k'}.
\end{equation}
All forward and backward variables can be computed iteratively using \eqref{eq:fwd_var_rec} and \eqref{eq:bwd_var_rec}. Having computed all forward and backward variables, the probability of the observed variable sequence is given by
\begin{align}
\prob(\mathbf{F}|\bm{\theta}) = \sum_{k=1}^{K}\alpha_{n,k}\beta_{n,k}. 
\end{align}
which holds for any $n\in\{1,\ldots,N\}$. Then the variables of interest, $q_{nk}$ and $\xi_n(k,k')$, can be obtained as 
\begin{equation}
q_{nk} = \prob(y_n = k|\mathbf{F},\bm{\theta}) = \frac{\alpha_{n,k} \beta_{n,k}}{\sum_{k'=1}^{K}\alpha_{n,k'} \beta_{n,k'} },
\end{equation}
\begin{align}
\xi_n (k,k') & = \prob(y_n = k, y_{n+1} = k' |\mathbf{F},\bm{\theta} ) \\ &  = \frac{\alpha_{n,k}T(k,k')b_{n+1,k'}\beta_{n+1,k'}}{\sum_{k',k^{''}=1}^{K}\alpha_{n,k'}T(k',k^{''})b_{n+1,k^{''}}\beta_{n+1,k^{''}} }.\notag
\end{align}
\ifCLASSOPTIONcaptionsoff
\newpage
\fi


%

\end{document}